\documentclass[11pt]{article}

\usepackage[utf8]{inputenc} 
\usepackage[T1]{fontenc}    
\usepackage{hyperref}       
\usepackage{url}            
\usepackage{booktabs}       
\usepackage{amsfonts}       
\usepackage{nicefrac}       
\usepackage{microtype}      
\usepackage{lmodern}

\usepackage{footmisc}
\usepackage{graphicx}
\usepackage{subcaption}
\usepackage{amsfonts}
\usepackage{amsmath,amssymb}
\usepackage{bbm}
\usepackage{xcolor}
\usepackage{fullpage}

\usepackage{enumitem}

\newcommand{\E}{\mathbf{E}}

\usepackage{xspace}

\usepackage{algorithm,algpseudocode}

\newcommand{\argmin}{\mathop{\arg\min}}
\DeclareMathOperator{\trace}{Tr}

\newtheorem{theorem}{Theorem}

\newtheorem{lemma}{Lemma}

\newtheorem{proof}{Proof}
\newtheorem{assumption}{Assumption}
\newtheorem{remark}{Remark}

\makeatletter
\let\@fnsymbol\@arabic
\makeatother

\title{Decentralised Learning with Random Features and Distributed Gradient Descent}

\author{Dominic Richards\footnote{Department of Statistics,   University of Oxford, 24-29 St Giles', Oxford, OX1 3LB}
  \qquad
  Patrick Rebeschini\footnotemark[1] 
  \qquad
  Lorenzo Rosasco\footnote{MaLGa Center, Universitá degli Studi di Genova, Genova, Italy}
  \footnote{Istituto Italiano di Tecnologia, Via Morego, 30, Genoa 16163, Italy} 
  \footnote{Massachusetts Institute of Technology, Cambridge, MA 02139, USA} \\
  \texttt{\{dominic.richards,patrick.rebeschini\}@stats.ox.ac.uk} \quad
  \texttt{lorenzo.rosasco@unige.it} 
}
\date{\today}

\begin{document}

\maketitle

\begin{abstract}
We investigate the generalisation performance of Distributed Gradient Descent with Implicit Regularisation and Random Features in the homogenous setting where a network of agents are given data sampled independently from the same unknown distribution. Along with reducing the memory footprint, Random Features are particularly convenient in this setting as they provide a common parameterisation across agents that allows to overcome previous difficulties in implementing Decentralised Kernel Regression. Under standard source and capacity assumptions, we establish high probability bounds on the predictive performance for each agent as a function of the step size, number of iterations, inverse spectral gap of the communication matrix and number of Random Features. By tuning these parameters, we obtain statistical rates that are minimax optimal with respect to the total number of samples in the network. The algorithm provides a linear improvement over single machine Gradient Descent in memory cost and, when agents hold enough data with respect to the network size and inverse spectral gap, a linear speed-up in computational runtime for any network topology. We present simulations that show how the number of Random Features, iterations and samples impact predictive performance.
\end{abstract}

\section{Introduction}

In supervised learning, an agent is given a collection of training data to fit a model that can predict the outcome of new data points. Due to the growing size of modern data sets and complexity of many machine learning models, a popular approach is to incrementally improve the model with respect to a loss function that measures the performance on the training data. The complexity and stability of the resulting model is then controlled implicitly by algorithmic parameters, such as, in the case of Gradient Descent, the step size and number of iterations. An appealing collection of models in this case are those associated to the Reproducing Kernel Hilbert Space (RKHS) for some positive definite kernel, as the resulting optimisation problem (originally over the space of functions) admits a tractable form through the Kernel Trick and Representer Theorem, see for instance \cite{scholkopf2001generalized}.

Given the growing size of data, privacy concerns as well as the  manner in which data is collected, distributed computation has become a requirement in many machine learning applications. Here training data is split across a number of agents which alternate between communicating model parameters to one another and performing computations on their local data. In centralised approaches (effective star topology), a single agent is typically responsible for collecting, processing and disseminating information to the agents. Meanwhile for many applications, including ad-hoc wireless and peer-to-peer networks, such centralised approaches are unfeasible. This motivates decentralised approaches where agents in a network only communicate locally within the network i.e. to neighbours at each iteration.

Many problems in decentralised multi-agent optimisation can be phrased as a form of consensus optimisation \cite{tsitsiklis1986distributed,tsitsiklis1984problems,johansson2007simple,nedic2009distributed,Nedic2009,johansson2009randomized,lobel2011distributed,matei2011performance,Boyd:2011:DOS:2185815.2185816,DAW12,shi2015extra,mokhtari2016dsa}. In this setting, a network of agents wish to minimise the average of functions held by individual agents, hence ``reaching consensus'' on the solution of the global problem. A standard approach is to augment the original optimisation problem to facilitate a decentralised algorithm. This typically introduces additional penalisation (or constraints) on the difference between neighbouring agents within the network, and yields a higher dimensional optimisation problem which decouples across the agents. This augmented problem can then often be solved using standard techniques whose updates can now be performed in a decentralised manner. While this approach is flexible and can be applied to many consensus optimisation problems, it often requires more complex algorithms which depend upon the tuning of additional hyper parameters, see for instance the Alternating Direction Method of Multiplers (ADMM) \cite{Boyd:2011:DOS:2185815.2185816}.

Many distributed machine learning problems, in particular those involving empirical risk minimisation, can been framed in the context of consensus optimisation. As discussed in \cite{bouboulis2017online,koppel2018decentralized}, for the case of Decentralised Kernel Regression it is not immediately clear how the objective ought to be augmented to facilitate both a decentralised algorithm and the Representer Theorem. Specifically, so the problem decouples across the network and agents have a common represention of the estimated function.  Indeed, while distributed kernel regression can be performed in the one-shot Divide and Conquer setting (Star Topology)   \cite{zhang2015divide,lin2017distributed,guo2017learning,mucke2018parallelizing,JMLR:v21:19-277} where there is a fusion center to combine the resulting estimators computed by each agent, in the decentralised setting there is no fusion center and agents must communicate for multiple rounds. A number of works have aimed to tackle this challenge \cite{forero2010consensus,mitra2014diffusion,gao2015diffusion,chouvardas2016diffusion,bouboulis2017online,koppel2018decentralized}, although these methods often include approximations whose impact on statistical performance is not clear\footnote{Additional details on some of these works have been included within Remark \ref{remark:PreviousLit} in the Appendix}. Most relevant to our work is \cite{bouboulis2017online} where Distributed Gradient Descent with Random Fourier Features is investigated in the online setting. In this case regret bounds are proven, but it is not clear how the number of Random Fourier Features or network topology impacts predictive performance in conjunction with non-parametric statistical assumptions\footnote{
We note the concurrent work \cite{xu2020coke} which also investigates Random Fourier Features for decentralised non-parametric learning. The differences from our work have been highlighted in Remark \ref{remark:COKE} in the Appendix.
}. 
For more details on the challenges of the developing a Decentralised Kernel Regression algorithm see Section \ref{sec:DecentralisedKR}.

\subsection{Contributions}
In this work we give statistical guarantees for a simple and practical Decentralised Kernel Regression algorithm. Specifically, we study the learning performance (Generalisation Error) of full-batch Distributed Gradient Descent \cite{nedic2009distributed} with implicit regularisation \cite{richards2018graph,richards2019optimal} and Random Features \cite{rahimi2008random,rudi2017generalization}. 
Random Features can be viewed as a form of non-linear sketching or shallow neural networks with random initialisations, and have be utilised to facilitate the large scale application of kernel methods by overcoming the memory bottle-neck.  In our case, they both decrease the memory cost and yield a simple Decentralised Kernel Regression algorithm. 
While previous approaches have viewed Decentralised Kernel Regression with explicit regularisation as an instance of consensus optimisation, where the speed-up in runtime depends on the network topology \cite{DAW12,scaman2017optimal}. We build upon \cite{richards2019optimal} and directly study the Generalisation Error of Distributed Gradient Descent with implicit regularisation. 
This allows linear speed-ups in runtime for \emph{any} network topology to be achieved by leveraging the statistical concentration of quantities held by agents. Specifically, our analysis demonstrates how the number of Random Features, network topology, step size and number of iterations impact Generalisation Error, and thus, can be tuned to achieve minimax optimal statistical rates with respect to all of the samples within the network \cite{caponnetto2007optimal}. When agents have sufficiently many samples with respect to the network size and topology, and the number of Random Features equal the number required by single machine Gradient Descent, a linear speed-up in runtime and linear decrease memory useage is achieved over single machine Gradient Descent. Previous guarantees given in consensus optimisation require the number of iterations to scale with the inverse spectral gap  of the network \cite{DAW12,scaman2017optimal}, and thus, a linear speed-up in runtime is limited to well connected topologies. We now provide a summary of our contributions.

\begin{itemize}
    \item \textbf{Decentralised Kernel Regression Algorithm}: By leveraging Random Features we develop a simple, practical and theoretically justified algorithm for Decentralised Kernel Regression. It achieves a linear reduction in memory cost and, given sufficiently many samples, a linear speed-up in runtime for any graph topology (Theorem \ref{thm:WorstCase}, \ref{thm:Refined}). This required extending the theory of Random Features to the decentralised setting (Section \ref{sec:ProofSketch}).
    \item \textbf{Refined Statistical Assumptions}: 
    Considering the attainable case in which the minimum error over the hypothesis class is achieved, we give guarantees that hold over a wider range of complexity and capacity assumptions. 
    This is achieved through a refined analysis of the Residual Network Error term (Section \ref{sec:ProofSketch:NetworkResidual}).
    \item \textbf{Bounds in High Probability}: All guarantees hold in high probability, where previous results \cite{richards2019optimal} for the decentralised setting only held in expectation. This is achieved through refined analysis of the Population Network Error  (Section \ref{sec:ProofSketch:NetworkPop}).
\end{itemize}

This work is structured as follows. Section \ref{sec:Setup} introduces the notation and Random Features. Section \ref{sec:MainResults} presents the main theoretical results. Section \ref{sec:ProofSketch} provides the error decomposition and a sketch proof of the refined analysis. Section \ref{sec:Experiments} presents simulation results. Section \ref{sec:Conclusion} gives the conclusion.

\section{Setup}
\label{sec:Setup}
This section introduces the setting. 
Section \ref{sec:DecentralisedKR} introduces Decentralised Kernel Regression and the challenges in developing a decentralised algorithm.
Section \ref{sec:Setup:FeatureMapsKernelMethods} introduces the link between Random Features and kernel methods. 
Section \ref{sec:Setup:DistGDRF} introduces Distributed Gradient Descent with Random Features.

\subsection{Challenges of Decentralised Kernel Regression}  
\label{sec:DecentralisedKR}
We begin with the single machine case then go on to the decentralised case.
\paragraph{Single Machine}Consider a standard supervised learning problem with squared loss. Given a probability distribution $\rho$ over $X \times \mathbb{R}$, we wish to solve
\begin{align}
\label{equ:Objective}
    \min_{f} \mathcal{E}(f), \quad 
    \mathcal{E}(f) = \int (f(x) - y)^2 d \rho(x,y),
\end{align}
given a collection of independently and identically distributed (i.i.d.) samples drawn from $\rho$, here denoted $(x_i,y_i)_{i=1}^{m} \in (X \times \mathbb{R}^{m})$. Kernel methods are non-parametric approaches defined by a kernel $k : X \times X \rightarrow \mathbb{R}$ which is symmetric and positive definite. The space of functions considered will be the Reproducing Kernel Hilbert Space associated to the kernel $k$, that is, the function space $\mathcal{H}$ defined as the completion of the linear span $\{K(x,\cdot): x \in X\}$ with respect to the inner product $\langle K(x,\cdot),K(x^{\prime},\cdot)\rangle_{\mathcal{H}} := K(x,x^{\prime})$ \cite{aronszajn1950theory}. When considering functions that minimise the empirical loss with explicit regularisation $\lambda \geq 0 $
\begin{align}
    \min_{f \in \mathcal{H}}
    \Big\{ 
    \frac{1}{m} \sum_{i=1}^{m}(f(x_i)- y_i)^2 
    + 
    \lambda \|f\|_{\mathcal{H}}^2
    \Big\}
\end{align}
we can appeal to the Representer Theorem \cite{scholkopf2001generalized}, and consider functions represented in terms of the data points, namely
$
    \widehat{f}(x) = \sum_{i=1}^{m} \alpha_i k(x_i,x)
$
where $\alpha = (\alpha_1,\dots,\alpha_m) \in \mathbb{R}^{m}$ are a collection of weights.
The weights are then often written in terms of the gram-matrix $K \in \mathbb{R}^{m \times m} $ whose $i,j$th entry is  $K_{ij} = k(x_i,x_j)$.

\paragraph{Decentralised} Consider a connected network of $n$ agents $G = (V,E)$ $|V| = n$, joined by edges $E \subseteq V \times V$, that wish to solve \eqref{equ:Objective}.  Each agent $v \in V$ has a collection of $m$ i.i.d. training points $(x_{i,v},y_{i,v})_{i=1}^{m} \in (X \times \mathbb{R})^{m}$ sampled from $\rho$. Following standard approaches in consensus optimisation we arrive at the optimisation problem 
\begin{align*}
    \min_{f_v \in \mathcal{H}, v \in V}
    \Big\{
    & \frac{1}{nm} \sum_{v \in V} \sum_{i=1}^{m} (f_v(x_{i,v}) - y_{i,v})^2 
    + \lambda \|f_v\|_{\mathcal{H} }^2
    \Big\} 
    \\
    & \quad 
    f_v = f_w \quad (v,w) \in E,
\end{align*}
where a local function for each agent $f_v$ is only evaluated at the data held by that agent $(x_{i,v},y_{i,v})_{i=1}^{m}$, and a constraint ensures agents that share an edge are equal. This constrained problem is then often solved by considering the dual problem \cite{scaman2017optimal} or introducing penalisation \cite{jakovetic2015linear}. 
In either case, the objective decouples so that given $\{f_v\}_{v \in V}$ it can be evaluated and optimised in a decentralised manner.  As discussed by \cite{bouboulis2017online,koppel2018decentralized}, it is not immediately clear whether a representation for $\{f_v\}_{v \in V}$ exists in this case that respects the gram-matrices held by each agent. Recall, in the decentralised setting, only agent $v$ can access the data $( x_{i,v} ,y_{i,v} )_{i=1}^{m}$ and the kernel evaluated at their data points $k(x_{i,v},x_{j,v})$ for $i,j =1,\dots,m$.

\subsection{Feature Maps and Kernel Methods}
\label{sec:Setup:FeatureMapsKernelMethods}
Consider functions parameterised by $\omega \in \mathbb{R}^{M}$ and written in the following form
\begin{align*}
    f(x) =  \langle \omega, \phi_{M}(x) \rangle, 
    \quad \forall x \in X,
\end{align*}
where $\phi_{M}: X \rightarrow \mathbb{R}^{M}, M \in \mathbb{N}$, denotes a family of finite dimensional feature maps that are identical and known across all of the agents. Feature maps in our case take a data point $x$ to a (often higher dimensional) space where Euclidean inner products approximate the kernel. That is, informally, 
$
    k(x,x^{\prime}) \approx \langle \phi_{M}(x),\phi_{M}(x^{\prime}) \rangle
$.
One now classical example is Random Fourier Features \cite{rahimi2008random} which approximate the Gaussian Kernel. 

\paragraph{Random Fourier Features} 
If $k(x,x^{\prime}) = G(x-x^{\prime})$, where $G(z) = e^{-\frac{1}{2\sigma^2} \|z\|^2}$, for $\sigma > 0$ then we have 
\begin{align*}
    & G(x - x^{\prime}) 
     = 
    \frac{1}{2\pi Z} \int \int_0^{2\pi}  
    \sqrt{2} \cos(\omega^{\top}x +b) \sqrt{2} \cos( \omega^{\top}x^{\prime} + b) e^{-\frac{\sigma^2}{2} \|\omega\|^2} d \omega db
\end{align*}
where $Z$ is a normalizing factor. Then, for the Gaussian kernel, $\phi_{M}(x) = M^{-1/2}(\sqrt{2} \cos(\omega_1^{\top}x + b_1),\dots,\sqrt{2} \cos(\omega_M^{\top} x + b_{M}))$, where $\omega_1,\dots,\omega_M$ and $b_1,\dots,b_M$ sampled independently from $\frac{1}{Z} e^{-\sigma^2 \|\omega\|^2/2}$ and uniformly in $[0,2\pi]$, respectively. 

More generally, this motivates the strategy in which we assume the kernel $k$ can be expressed as  
\begin{align}
\label{equ:KernelIntegral}
    k(x,x^{\prime}) = \int \psi(x,\omega) \psi(x^{\prime},\omega)d\pi(\omega), 
    \quad 
    \forall x,x^{\prime} \in X,
\end{align}
where $(\Omega,\pi)$ is a probability space and $\psi:X \times \Omega \rightarrow \mathbb{R}$ \cite{reed2012methods}. Random Features can then be seen as Monte Carlo approximations of the above integral.

\subsection{Distributed Gradient Descent and Random Features}
\label{sec:Setup:DistGDRF}
Since the functions are now linearly parameterised by $\omega \in \mathbb{R}^{M} $, agents can consider the simple primal method Distributed Gradient Descent \cite{nedic2009distributed}. Initialised at $\widehat{\omega}_{1,v} = 0;$ for $v \in V$, agents update their iterates for $t \geq 1$
\begin{align}
\label{alg:DistributedGD}
    & \widehat{\omega}_{t+1,v} = \sum_{w \in V}
    P_{vw}  \Big( \widehat{\omega}_{t,w}
    - 
    \frac{\eta}{m} \sum_{i=1}^{m} \big( \langle \widehat{\omega}_{t,w},\phi_{M}(x_{i,w}) \rangle - y_{i,w}\big) \phi_{M}(x_{i,w})
    \Big),\nonumber
\end{align}
where $P \in \mathbb{R}^{n \times n}$ is a doubly stochastic matrix supported on the network i.e. $P_{ij} \not= 0 $ only if $(i,j) \in E$, and $\eta$ is a fixed stepsize. The above iterates are a combination of two steps. Each agent performing a local Gradient Descent step with respect to their own data i.e. $\widehat{\omega}_{t,w}- \frac{\eta}{m} \sum_{i=1}^{m} \big( \langle \widehat{\omega}_{t,w},\phi_{M}(x_{i,w}) \rangle - y_{i,w}\big) \phi_{M}(x_{i,w})$ for agent $w \in V$. And a communication step where agents average with their neighbours as encoded by the summation $\sum_{w \in V}P_{vw} a_{w}$, where $a_w$ is the quantity held by agent $w \in V$. The performance of Distributed Gradient Descent naturally depends on the connectivity of the network. In our case it is encoded by the second largest eigenvalue of $P$ in absolute value, denoted $\sigma_2 \in [0,1)$. In particular, it arises through the inverse spectral gap  $1/(1-\sigma_2)$, which is known to scale with the network size for particular topologies, that is $O\big( 1/(1-\sigma_2) \big)  = O(n^{\beta})$ where $\beta =2$ for a cycle, $\beta = 1$ for a grid and $\beta  = 0$ for an expander, see for instance \cite{DAW12}.  Naturally, more ``connected'' topologies have larger spectral gaps, and thus, smaller inverses. 
\paragraph{Notation} For $a,b \in \mathbb{R}$ we denote $a \vee b$ as the maximum between $a$ and $b$  and $a \wedge b$ the minimum. We say $a \simeq b$ if there exists a constant $c$ independent of $n,m,M,(1-\sigma_2)^{-1}$ up-to logarithmic factors such that $a = c b$. Similarly we write $a \lesssim b$ if $a \leq b c$  and $a \gtrsim b$ if $a \geq c b$.

\section{Main Results}
\label{sec:MainResults}
This section presents the main results of this work. Section \ref{sec:MainResults:WorstCase} provides the results under basic assumptions. Section \ref{sec:MainResults:Refined} provides the results under more refined assumptions.

\subsection{Basic Result}
\label{sec:MainResults:WorstCase}
We begin by introducing the following assumption related to the feature map. 
\begin{assumption}
\label{ass:RandomFeatures}
Let $(\Omega,\pi)$ be a probability space and define the feature map $\psi: X \times \Omega \rightarrow \mathbb{R}$  for all $x \in X$ such that \eqref{equ:KernelIntegral} holds. Define the family of feature maps for $M > 0 $
\begin{align*}
    \phi_{M}(x) := \frac{1}{\sqrt{M}}(\psi(x,\omega_1),\dots,\psi(x,\omega_{M}))
\end{align*}
where $(\omega_j)_{j=1}^{M} \in \Omega$ are sampled independently from $\pi$. 
\end{assumption}
The above assumption states that the feature map is made of $M$ independent features $\psi(x,\omega_{i})$ for $i=1,\dots,M$.  This is satisfied for a wide range of kernels, see for instance Appendix E of \cite{rudi2017generalization}. The next assumption introduces some regularity to the feature maps. 
\begin{assumption}
\label{ass:FeatureRegularity}
The function $\psi$ is continuous and there exists $\kappa \geq 1$ such that $|\psi(x,\omega)| \leq \kappa$ for any $x \in X, \omega \in \Omega$.
\end{assumption}
This implies that the kernel considered is bounded $|k(x,x^{\prime})| \leq \kappa^2$ which is a common assumption in statistical learning theory \cite{cucker2007learning,steinwart2008support}. The following assumption is related to the optimal predictor. 
\begin{assumption}
\label{ass:RKHS}
Let $\mathcal{H}$ be the RKHS with kernel $k$. Suppose there exists $f_{\mathcal{H}} \in \mathcal{H}$ such that 
$
    \mathcal{E}(f_{\mathcal{H}}) = \inf_{f \in \mathcal{H}} \mathcal{E}(f)
$.
\end{assumption}
It states that the optimal predictor is within the interior of $\mathcal{H}$. Moving beyond this assumption requires considering the non-attainable case, see for instance \cite{dieuleveut2016nonparametric},  which is left to future work. Finally, the following assumption is on the response moments. 
\begin{assumption}
\label{ass:moment}
For any $x \in X$ 
\begin{align*}
    \int y^{2\ell} d \rho(y|x) \leq \ell ! B^{\ell} p,
    \quad 
    \forall \ell \in \mathbb{N}
\end{align*}
for constants $B \in (0,\infty)$ and $p \in (1,\infty)$, $\rho_{X}-$almost surely. 
\end{assumption}
This assumption is satisfied if the response is bounded or generated from a model with independent zero mean Gaussian noise. 

Given an estimator $\widehat{f}$, its excess risk is defined as $\mathcal{E}(\widehat{f}) - \mathcal{E}(f_{\mathcal{H}})$. Let the estimator held by agent $v \in V$ be denoted by $\widehat{f}_{t,v} = \langle \widehat{\omega}_{t,v},\phi_{M}(\cdot) \rangle$, where $\widehat{\omega}_{t,v}$ is the output of Distributed Gradient Descent \eqref{alg:DistributedGD} for agent $v$. Given this basic setup, we state the prediction bound prescribed by our theory. 
\begin{theorem}[Basic Case]
\label{thm:WorstCase}
Let $n,m,M \in \mathbb{N}_{+}$, $\delta \in (0,1)$, $t \geq 4$, $\eta \kappa^2 \leq 1$ and $\eta \simeq 1$. 
Under assumptions \ref{ass:RandomFeatures} to \ref{ass:moment}, the following holds with high probability for any $v \in V$
\begin{align*}
    \mathcal{E}(\widehat{f}_{t+1,v}) - \mathcal{E}(f_{\mathcal{H}}) 
    \lesssim \frac{1}{\sqrt{nm}}
\end{align*}
\vspace{-0.25cm}
when
\begin{align}
     m  \gtrsim 
     \frac{n^{3}}{(1-\sigma_2)^{4}} ,
     \, \text{  } \,
     M  \simeq \sqrt{nm},
     \, \text{ and }\,
     t = \sqrt{nm}.
\end{align}
\end{theorem}
Theorem \ref{thm:WorstCase} demonstrates that Distributed Gradient Descent with Random Features achieves optimal statistical rates, in the minimax sense \cite{caponnetto2007optimal,blanchard2018optimal}, with respect to all $nm$ samples when three conditions are met. The first $m  \gtrsim  n^{3}/(1-\sigma_2)^{4}$ ensures that the network errors, due to agents communicating locally on the network, are sufficiently small from the phenomena of concentration. The second $M  \simeq \sqrt{nm} $ ensures that the agents have sufficiently many Random Features to control the kernel approximation. It aligns with the number required by single machine Gradient Descent with all $nm$ samples \cite{carratino2018learning}. Finally $t = \sqrt{nm}$ is the number of iterations required to trade off the bias and variance error terms. This is the number of iterations required by single machine Gradient Descent with all $nm$ samples, and thus, due to considering a distributed algorithm, gives a linear speed-up in runtime. We now discuss the runtime and space complexity of Distributed Gradient Descent with Random Features when the covariates take values in $\mathbb{R}^{D}$ for some $D > 0$. Remark \ref{remark:communication} in Appendix \ref{sec:Remarks} shows how, with linear features, Random Features can yield communication savings when $D> M$.

\paragraph{Pre-processing + Space Complexity} After a pre-processing step which costs $O(DMm) = O(Dm^{3/2} \sqrt{n})$, Distributed Gradient Descent has each agent store a $m \times M = m \times \sqrt{nm}$ matrix. Single machine Gradient Descent performs a $O(DM nm) = O(D(nm)^{3/2})$ pre-processing step and stores a $nm \times M = nm \times \sqrt{nm}$ matrix. Distributed Gradient Descent thus gives a linear order $n$ improvement in pre-processing time and memory cost. 

\paragraph{Time Complexity} Suppose one gradient computation costs 1 unit of time and communicating with neighbours costs $\tau$. Given sufficiently many samples  $m  \gtrsim n^{3}/(1-\sigma_2)^{4}$ then \emph{Single Machine Iterations} = \emph{Distributed Iterations} and the speed-up in runtime for Distributed Gradient Descent over single machine Gradient Descent is 
\begin{align*}
     \text{Speed-up} 
    &  := 
    \frac{\text{Single Machine Runtime}}{\text{Distributed Runtime}} 
    = 
    \frac{\text{Single Machine Iteration Time }}{\text{Distributed Iteration Time }} 
    \underbrace{ 
    \frac{\text{Single Machine Iters. }}{\text{Distributed Iters. }}
    }_{ = 1} \\
    & = 
    \frac{ nm }{m + \tau + M \text{Deg}(P) } \simeq n 
\end{align*} 
where the final equality holds when the communication delay and cost of aggregating the neighbours solutions is bounded $\tau + M \text{Deg}(P) \lesssim m$. This observation demonstrates a linear speed-up in runtime can be achieved for \emph{any} network topology. This is in contrast to results in decentralised consensus optimisation where the speed-up in runtime usually depends on the network topology, with a linear improvement only occurring for well connected topologies i.e. expander and complete, see for instance \cite{DAW12,scaman2017optimal}.

\subsection{Refined Result}
\label{sec:MainResults:Refined}
Let us introduce two standard statistical assumptions related to the underlying learning problem. With the marginal distribution on covariates $\rho_{X}(x) := \int_{\mathbb{R}} \rho(x,y) dy$ and the space of square integrable functions $L^2(X,\rho_{X}) = \{ f: X \rightarrow \mathbb{R} : \|f\|_{\rho}^{2} = \int |f|^2 d \rho_{X} < \infty \}$, let $L: L^2(X,\rho_X) \rightarrow L^2(X,\rho_{X})$ be the integral operator defined for $x \in X$ as
$
    Lf(x) = \int k(x,x^{\prime}) f(x^{\prime}) d \rho_{X}(x^{\prime}), \, 
    \forall f \in L^2(X,\rho_{X}) 
$. 
The above operator is symmetric and positive definite. The assumptions are then as follows. 
\begin{assumption}
\label{ass:SourceCap}
For any $\lambda > 0$, define the effective dimension as $\mathcal{N}(\lambda) := \trace\big( \big( L + \lambda I)^{-1} L \big)$, and assume there exists $Q > 0 $ and $\gamma \in [0,1]$ such that 
$
    \mathcal{N}(\lambda) \leq Q^2 \lambda^{-\gamma}
$.\\ 
Moreover, assume there exists  $1 \geq r \geq 1/2$ and $g \in L^2(X,\rho_{X})$ such that 
$
    f_{\mathcal{H}}(x) = (L^r g)(x)
$.
\end{assumption}
The above assumptions will allow more refined bounds on the Generalisation Error to be given. The quantity $\mathcal{N}(\lambda)$ is the effective dimension of the hypothesis space, and Assumption \ref{ass:SourceCap} holds for $\gamma > 0$ when the $i$th eigenvalue of $L$ is of the order $i^{-1/\gamma}$, for instance. Meanwhile, the second condition for $1 \geq r \geq 1/2$ determines which subspace the optimal predictor is in. Here larger $r$ indicates a smaller sub-space and a stronger condition. The refined result is then as follows.
\begin{theorem}[Refined]
\label{thm:Refined}
Let $n,m,M \in \mathbb{N}_{+}$, $\delta \in (0,1)$, $t \geq 2t^{\star} \geq 4$, $\eta \kappa^2 \leq 1$ and $\eta \simeq 1$. 
Under assumptions \ref{ass:RandomFeatures} to \ref{ass:SourceCap} with $r+\gamma > 1$, the following holds with high probability for any $v \in V$
\begin{align*}
    \mathcal{E}(\widehat{\omega}_{t+1,v}) - \mathcal{E}(f_{\mathcal{H}})
    \lesssim 
    (nm)^{\frac{-2r}{2r+\gamma}}
\end{align*}
when we let $t^{\star}  \simeq  1/(1 - \sigma_2)  $ and have 
\begin{align*}
    & 
    \underbrace{ 
    m  \gtrsim 
    \Big( 
     (t^{\star})^{\frac{(1+\gamma)(2r+\gamma)}{2(r+\gamma - 1) } } n^{\frac{r+1}{r+\gamma - 1}} 
    \Big)
    \vee
    \Big( 
    (t^{\star})^{2 \vee (2r+\gamma) } n^{\frac{2r}{\gamma}} 
    \Big)
    }_{\text{Sufficiently Many Samples}}
    \\
 & \underbrace{ 
    M  \simeq (nm)^{\frac{ 1 + \gamma(2r-1)}{2r+\gamma}} 
    }_{\text{Single Machine Random Features}}
    \quad  
    \underbrace{ t = (nm)^{\frac{1}{2r+\gamma}} }_{ \text{Single Machine Iterations} } 
\end{align*}
\end{theorem}
Once again, the statistical rate achieved $ (nm)^{-\frac{2r}{2r+\gamma}}$ is the minimax optimal rate with respect to all of the samples within the network \cite{caponnetto2007optimal}, and both the number of Random Features as well as the number of iterations match the number required by single machine Gradient Descent when given \textit{sufficiently many samples } $m$. When $r=1/2$ and $\gamma =1$ we recover the basic result given in Theorem \ref{thm:WorstCase}, with the bounds now adapting to complexity of the predictor as well as capacity through $r$ and $\gamma$, respectively. In the low dimensional setting when $\gamma = 0$, we note our guarantees do not offer computational speed-ups over single machine Gradient Descent. While counter-intuitive, this observation aligns with \cite{richards2019optimal}, which found the easier the problem (larger $r$, smaller $\gamma$) the more samples required to achieve a speed-up. This is due to network error concentrating at fixed rate of $1/m$ while the optimal statistical rate is $(nm)^{-\frac{2r}{2r+\gamma}}$. An open question is then how to modify the algorithm to exploit regularity and achieve a speed-up runtime, similar to how Leverage Score Sampling exploits additional regularity \cite{bach2013sharp,pmlr-v70-avron17a,rudi2018fast,li2018towards}.

To provide insight into how the conditions in Theorem \ref{thm:Refined} arise, the following theorem gives the leading order error terms which contribute to the conditions in Theorem \ref{thm:Refined}.

\begin{theorem}[Leading Order Terms]
\label{thm:LeadingOrder}
Let $n,m,M \in \mathbb{N}_{+}$, $\delta \in (0,1)$, $t \geq 2t^{\star} \geq 4$, $\eta \kappa^2 \leq 1$ and $\eta \simeq 1$. 
Under assumptions \ref{ass:RandomFeatures} to \ref{ass:SourceCap} with $r+\gamma > 1$, the following holds with high probability when 
$t^{\star}  \simeq \frac{1}{1 - \sigma_2}  $ for any $ v\in V$
\begin{align*}
     \mathcal{E}(\widehat{f}_{t+1,v}) - \mathcal{E}(f_{\mathcal{H}})
    &  \lesssim 
    \quad 
    \underbrace{ 
    \frac{ \eta ^{\gamma}}{m (1-\sigma_2)^{\gamma} }
    + 
    \frac{(\eta t)^{2} (\eta t^{\star})^{1+\gamma} }{m^2}
    }_{\text{Network Error}} 
     + \, 
    \underbrace{ 
    \Big( \frac{\eta t }{M} \! + \! 1 \Big) 
    \frac{ (\eta t)^{\gamma} }{nm}
    + 
    \frac{ 1}{M (\eta t)^{(1-\gamma)(2r-1)}}
    + 
    \big( \frac{1}{\eta t}\big)^{2r}
    }_{\text{Statistical Error}}\\
    & \quad\quad + \text{H.O.T.}
\end{align*}
where H.O.T. denotes Higher Order Terms. 
\end{theorem}
Theorem \ref{thm:LeadingOrder} decomposes the Generalisation Error into two terms. The \textit{Statistical Error} matches the Generalisation Error of Gradient Descent with Random Features \cite{carratino2018learning} and consists of Sample Variance, Random Feature and Bias errors. The  \textit{Network Error} arises from tracking the difference between the Distributed Gradient Descent $\widehat{\omega}_{t+1,v}$ and single machine Gradient Descent iterates. The primary technical contribution of our work is in the analysis of this term, in particular, building on \cite{richards2019optimal} in two directions. Firstly, bounds are given in high probability instead of expectation. Secondly, we give a tighter analysis of the Residual Network Error, here denoted in the second half of the \textit{Network Error} as  $(\eta t)^{2} (\eta t^{\star})^{1+\gamma}/m^2$. Previously this term was of the order $(\eta t)^{2+\gamma}/m^2$ and gave rise to the condition of $r+\gamma/2 \geq 1$, whereas we now require $r + \gamma \geq 1$. Our analysis can ensure it is decreasing with the step size $\eta$, and thus, be controlled by taking a smaller step size. While not explored in this work, we believe our approach would be useful for analysing the Stochastic Gradient Descent variant \cite{lin2017optimal} where a smaller step size is often chosen.

\section{Error Decomposition and Proof Sketch}
\label{sec:ProofSketch}
In this section we give a more detailed error decomposition as well as a sketch of the proof. 
Section \ref{sec:ProofSketch:ErrorDecomp} gives the error decomposition into statistical and network terms. 
Section \ref{sec:ProofSketch:NetworkError} decomposes the network term into a population and a residual part. 
Section \ref{sec:ProofSketch:NetworkPop} and \ref{sec:ProofSketch:NetworkResidual} give sketch proofs for bounding the population and residual parts respectively.

\subsection{Error Decomposition}
\label{sec:ProofSketch:ErrorDecomp}
We begin by introducing the iterates produced by a single machine Gradient Descent with $nm$ samples as well as an auxiliary sequence associated to the population. Initialised at $\widehat{v}_{1} = \widetilde{v}_{1} = 0$, we  define, for $t \geq 1$
\begin{align*}
    & \widehat{v}_{t+1}  \>\!\!\! = \! 
    \widehat{v}_{t} \! - \! 
    \frac{\eta}{nm} \!
    \sum_{w \in V}
    \!
    \sum_{i=1}^{m}
    \>\!\!
    \big( \langle \widehat{v}_{t,w},
    \phi_{M}(\>\!\! x_{i,w} \>\!\!) \rangle 
    \! - \! y_{i,w}\big) 
    \phi_{M}( \>\!\! x_{i,w} \>\!\! ),
    \\
    & \widetilde{v}_{t+1} = \widetilde{v}_{t} - \eta 
    \int_{X} \big( \langle \widetilde{v}_{t},\phi_{M}(x) \rangle - y )\phi_{M}(x) d \rho(x,y).
\end{align*}
We work with functions in $L^2(X,\rho_{X})$, thus we define
$\widehat{g}_{t} = \langle \widehat{v}_{t},\phi_{M}(\cdot) \rangle$, $\widetilde{g}_{t} = \langle \widetilde{v}_{t},\phi_{M}(\cdot) \rangle$.
Since the prediction error can be written in terms of the $L^2(X,\rho_{X})$ as follows
$
    \mathcal{E}(\widehat{f}_{t,v}) - \mathcal{E}(f_{\mathcal{H}}) = \|\widehat{f}_{t,v} - f_{\mathcal{H}}\|_\rho^2
$
we have the decomposition 
$
    \widehat{f}_{t,v} - f_{\mathcal{H}} 
    = 
    \widehat{f}_{t,v} - \widehat{g}_{t} 
    + 
    \widehat{g}_{t} - f_{\mathcal{H}}
$. The term $\widehat{g}_{t} - f_{\mathcal{H}}$ that we call the \textit{Statistical Error} is studied within \cite{carratino2018learning}. The primary contribution of our work is in the analysis of $\widehat{f}_{t,v} - \widehat{g}_{t}$ which we call the \textit{Network Error}, and go on to describe in more detail next.

\subsection{Network Error}
\label{sec:ProofSketch:NetworkError}
To accurately describe the analysis for the network error we introduce some notation. Begin by defining the operator $S_{M}: \mathbb{R}^{M}\rightarrow L^2(X,\rho_{X})$ so that $(S_{M}\omega)(\cdot) = \langle \omega,\phi_{M}(\cdot)\rangle$  as well as the covariance $C_M: \mathbb{R}^{M} \rightarrow \mathbb{R}^{M}$ defined as $C_{M} = S^{\star}_{M} S_{M}$, where $S^{\star}_{M}$ is the adjoint of $S_{M}$ in $L^2(X,\rho_{X})$. Utilising an isometry property (see \eqref{equ:Isometry} in the Appendix) we have for $\omega \in \mathbb{R}^{M}$ the following $\|S_{M} \omega \|_{\rho} = \|C_{M}^{1/2} \omega\|$, that is going from a norm in $L^{2}(X,\rho_{X})$ to Euclidean norm. The empirical covariance operator of the covariates held by agent $v \in V$ is denoted  $\widehat{C}_{M}^{(v)}:\mathbb{R}^{M} \rightarrow \mathbb{R}^{M}$. For $t \geq 1$ and a path $w_{t:1} = (w_t,w_{t-1},\dots,w_{1}) \in V^{t}$ denote the collection of contractions 
\begin{align*}
    \Pi(w_{t:1}) = (I - \eta \widehat{C}_{M}^{(w_t)})(I - \eta \widehat{C}_{M}^{(w_{t-1})})\dots (I - \eta \widehat{C}_{M}^{(w_{1})})
\end{align*}
as well as the centered product $\Pi^{\Delta}(w_{t:1})= \Pi(w_{t:1}) - (I - \eta C_M)^{t}$. For $w \in V$  $k \geq 1$ let $N_{k,w} \in \mathbb{R}^{M}$ denote a collection of zero mean random variables that are independent across agents $w \in V$ but not index $k \geq 1$.

For $v,w \in V$ and $s \geq 1$ define the difference $\Delta^{s}(v,w) := P^{s}_{vw} - \frac{1}{n}$, where we apply the power then index i.e. $(P^{s})_{vw} = P^{s}_{vw}$. For $w_{t:k} \in V^{t-k}$ denote the deviation along a path $\Delta(w_{t:k}) = P_{v w_{t:k}} - \frac{1}{n^{t-k}}$ where we have written the probability for a path $P_{v w_{t:k}} = P_{vw_t} P_{w_{t} w_{t-1}}\dots P_{w_{k+1} w_{k}}$.

Following \cite{richards2019optimal}, center the distributed $\omega_{t+1,v}$ and the single machine iterates $\widehat{v}_{t+1}$ around the population iterates $\widetilde{v}_{t}$. Apply the isometry property to $\|\widehat{f}_{t,v} - \widehat{g}_{t} \|_{\rho} = \|C_{M}^{1/2}( \widehat{\omega}_{t+1,v} - \widehat{v}_{t})\|$ and following the steps in Appendix \ref{App:Network:ErrorDecomp} we arrive at 
\begin{align*}
    \|C_{M}^{1/2}(\widehat{\omega}_{t+1,v} - \widehat{v}_{t+1})\| 
     & \leq   \underbrace{ 
    \sum_{k=1}^{t} \eta 
    \sum_{w \in V}| \Delta^{t-k}(v,w)|
    \|C_{M}^{1/2}(I - \eta C_{M})^{t-k} N_{k,w}\|
    }_{\text{Population Network Error}} 
    \\
    & \quad 
    + \underbrace{ 
    \sum_{k=1}^{t} \eta 
    \big\| \sum_{w_{t:k} \in V^{t-k+1}} \Delta(w_{t:k})C_{M}^{1/2} \Pi^{\Delta}(w_{t:k+1}) N_{k,w_{k}} \big\|.
    }_{\text{Residual Network Error}}
\end{align*}
The two terms above can be associated to the two terms in the network error of Theorem \ref{thm:LeadingOrder}, with the Population Network Error decreasing as $1/m$ and the Residual Network Error as $1/m^2$. We now analyse each of these terms separately.   

\subsection{Network Error: Population}
\label{sec:ProofSketch:NetworkPop}
Our contribution for analysing the \textit{Population Network Error} is to give bounds it in high probability, where as \cite{richards2019optimal} only gave bounds in expectation. Choosing some $t \geq 2 t^{\star} \geq 2$ and splitting the series at $k = t-t^{\star}$ we are left with two terms. For $1 \leq k \leq t-t^{\star}$ we utilise that the sum over the difference $|\Delta^{s}(v,w)|$ can be written in terms of euclidean $\ell_1$ norm and this is bounded by the second largest eigenvalue of $P$ in absolute value i.e. $\sum_{w \in V}| \Delta^{t-k}(v,w)| = \|e_v^{\top}P^{t-k} - \frac{1}{n} \mathbf{1}\|_1 \leq \sqrt{n}\sigma_2^{t-k} \leq \sqrt{n} \sigma_2^{t^{\star}}$, where $e_v$ is the standard basis vector in $\mathbb{R}^{n}$ with a $1$ aligning with agent $v \in V$ and $\mathbf{1}$ is a vector of all $1$'s. Meanwhile for $t \geq k \geq t-t^{\star}$, we follow \cite{richards2019optimal} and utilise the contraction of the gradient updates i.e. $C_{M}^{1/2}(I - \eta C_{M})^{t-k}$ alongside that $N_{k,w_{k}}$ is an average of $m$ i.i.d. random variables, and thus, concentrate at $1/\sqrt{m}$ in high probability. This leads to the bound in high probability
\begin{align*}
    & \text{Population Network Error} \lesssim
    \!\!
    \underbrace{ 
    \frac{\sqrt{n} \sigma_2^{t^{\star}}t }{\sqrt{m}}
    }_{\text{Well Mixed Terms}} + 
    \underbrace{ 
    \frac{(\eta t^{\star})^{\gamma/2}}{\sqrt{m}}
    .}_{\text{Poorly Mixed Terms}} 
\end{align*}
The first term \textit{Well Mixed}, decays exponentially with the second largest eigenvalue of $P$ in absolute value, and represents the information from past iterates that has now fully propagated around the network. The term \textit{Poorly Mixed} represents error from the most recent iterates that is yet to fully propagate through the network. It grows at the rate $(t^{\star})^{\gamma/2}$ due to utilising the contractions of the gradients as well as the assumptions \ref{ass:SourceCap}.  The quantity $t^{\star}$ is now chosen to trade off these terms. Note by writing $\sigma_2^{t^{\star}} = e^{-t^{\star} \log(1/\sigma_2)}$ that, up to logarithmic factors, the first can be made small by taking $t^{\star} \gtrsim \frac{1}{1-\sigma_2} \geq \frac{1}{-\log(\sigma_2)}$.

\subsection{Network Error: Residual}
\label{sec:ProofSketch:NetworkResidual}
The primary technical contribution of our work is in the analysis of this term. The analysis builds on insights from \cite{richards2019optimal}, specifically that $\Pi^{\Delta}(w_{t:1})$ is a product of empirical operators minus the population, and thus, can be written in terms of the differences $\widehat{C}_{M}^{(w)} - C_{M}$ which concentrate at $1/\sqrt{m}$. Specifically, for $N \in \mathbb{R}^{M}$, the bound within \cite{richards2019optimal} was of the following order with high probability for any $w_{t:1} \in V^{t}$
\begin{align}
\label{equ:DeltaOperatorBound}
    \|C_{M}^{1/2} \Pi^{\Delta}(w_{t:1}) N\| \lesssim \|N\| \frac{(\eta t)^{\gamma/2}}{\sqrt{m}}.
\end{align}
The bound for Residual Network Error within \cite{richards2019optimal}  is arrived at by applying triangle inequality over the series $\sum_{w_{t:k} \in V^{t-k+1}}$, plugging in \eqref{equ:DeltaOperatorBound} for $\|C_{M}^{1/2} \Pi^{\Delta}(w_{t:k+1}) N_{k,w_{k}}\|$ alongside $\|N_{k,w_{k}}\| \lesssim 1/\sqrt{m}$ see Lemma \ref{lem:conc} in Appendix. Summing over $1 \leq k \leq t$ yields the bound of order $(\eta t)^{1+\gamma/2}/m$ in high probability. The two key insights of our analysis are as follows. Firstly, noting that the error for bounding the contraction  $\Pi^{\Delta}(w_{t:1})$ grows with the length of the path, and as such, we should aim to apply the bound \eqref{equ:DeltaOperatorBound} to short paths. Secondly, note for $N \in \mathbb{R}^{M}$ quantities of the form $\|C_{M}^{1/2}\sum_{w_{t:1} \in V^{t}} \Delta(w_{t:1}) \Pi^{\Delta}(w_{t:1}) N\|$ concentrate quickly (Lemma \ref{lem:OperatorNormBound:Delta2} in Appendix). 

To apply the insights outlined previously, we decompose the deviation $\Pi^{\Delta}(w_{t:2})$ into two terms that only replace the final $t^{\star}$ operators with the population, that is 
\begin{align*}
    \Pi^{\Delta} (  w_{t:2}  ) 
     =   
    \Pi( w_{t:t^{\star} +  2}) \Pi^{\Delta}
    ( w_{t^{\star} +  1 :1}  ) 
     + 
    \Pi^{\Delta}(w_{t:t^{\star}  +  2}) 
    ( I  - \eta C_{M}  )^{t^{\star}}.
\end{align*}
Plugging in the above then yields, for the case $k=1$, 
\begin{align*}
     \sum_{w_{t:1} \in V^{t}} \Delta(w_{t:1})C_{M}^{1/2} \Pi^{\Delta}(w_{t:2}) N_{k,w_{1}} 
    & =
    \sum_{w_{t:1} \in V^{t}} \Delta(w_{t:1})C_{M}^{1/2} \Pi(w_{t:t^{\star} + 2}) \underbrace{ \Pi^{\Delta}(w_{t^{\star} +1 :1})}_{t^{\star} \text{contraction}}   N_{k,w_{1}}\\
    & \, + 
    \sum_{w_{t:1} \in V^{t}} 
    \Delta(w_{t:1})
    \underbrace{ 
    C_{M}^{1/2} \Pi^{\Delta}(w_{t:t^{\star} +2}) 
     (I - \eta C_{M})^{t^{\star}}  N_{k,w_{1}}
     }_{\text{Independent of $w_{t^{\star} + 1:1}$ }}
\end{align*}
Note that the first term above only contains a contraction $\Pi^{\Delta}(w_{t^{\star} +1 :1})$ of length $t^{\star}$, and as such, when applying a variant of \eqref{equ:DeltaOperatorBound} will only grow at length $(\eta t^{\star})^{(1+\gamma)/2}/\sqrt{m}$. When summing over $1 \leq k \leq t$ this will result in the leading order term for the residual error of $(\eta t) (\eta t^{\star})^{(1+\gamma)/2}/m$. For the second term, note the highlighted section is independent of the final $t^{\star}$ steps of the path $w_{t:1}$, namely $w_{t^{\star} +1:1}$. Therefore we can sum the deviation $\Delta(w_{t:1})$ over path $w_{t^{\star} +1:1}$ and, if $t^{\star} \gtrsim \frac{1}{1-\sigma_2}$, replace $N_{k,w_{1}}$ by the average $\frac{1}{n} \sum_{w \in V} N_{k,w}$. This has impact of decoupling the summation over the remainder of the path $w_{t:t^{\star}}$ allowing the second insight from previously to be used. For details on this step we point the reader to Appendix Section \ref{App:Network:ErrorDecomp}.

\section{Experiments}
\label{sec:Experiments}
For our experiments we consider subsets of the SUSY data set \cite{baldi2014searching}, as well as single machine and Distributed Gradient Descent with a fixed step size $\eta = 1$. Cycle and grid network topologies are studied, with the matrix $P$ being a simple random walk. Random Fourier Features are used $\psi(x,\omega) = \cos( \xi \times w^{\top} x + q)$, with $\omega:= (w,q)$, $w$ sampled according to the normal distribution, $q$ sampled uniformly at random between $0$ and $2 \pi$, and $\xi$ is a tuning parameter associated to the bandwidth (fixed to $\xi = 10^{-1/2}$). For any given sample size, topology or network size we repeated the experiment 5 times. Test size of $10^{4}$ was used and classification error is minimum over iterations and maximum over agents i.e. $\min_{t} \max_{v \in V} \mathcal{E}_{\mathrm{Approx}}(\widehat{\omega}_{t,v})$, where $\mathcal{E}_{\mathrm{Approx}}$ is approximated test error. With the response of the data being either 1 or 0 and the predicted response $\widehat{y}$, the predicted classification is the indicator function of $\widehat{y} > 1/2$. The classification error is the proportion of mis-classified samples.

We begin by investigating the number of Random Features required with Distributed Gradient Descent to match the single machine performance. 
Looking to Figure \ref{fig:Plots1}, observe that for a grid topology, as well as small cycles $(n=9,25)$, that the classification error aligns with a single machine beyond approximately $\sqrt{nm}$ Random Features. For larger more poorly connected topologies, in particular a cycle with $n=49$ agents, we see that the error does not fully decrease down that of single machine Gradient Descent. 
\begin{figure}[!h]
    \centering
    \includegraphics[width=0.35\textwidth]{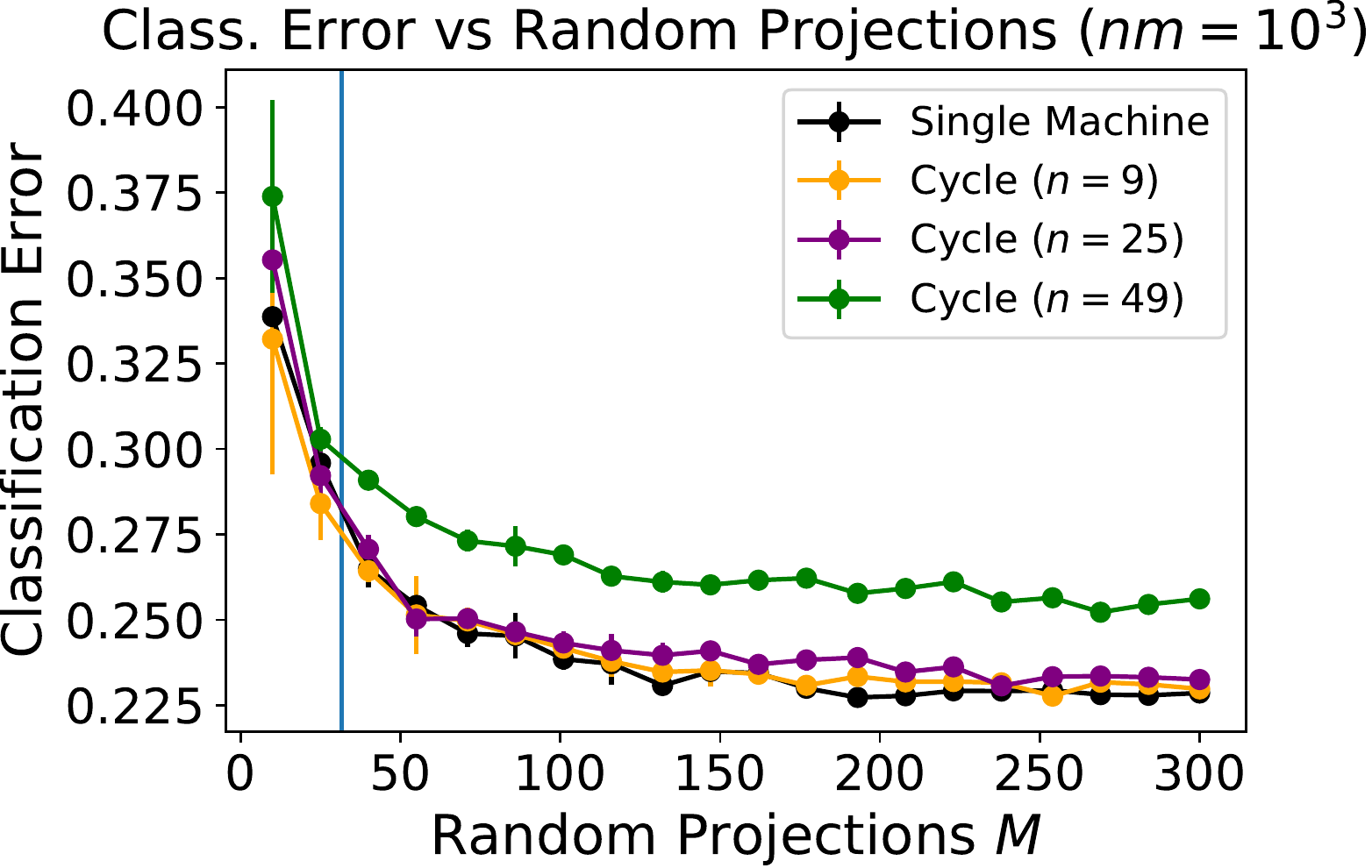}
    \includegraphics[width=0.35\textwidth]{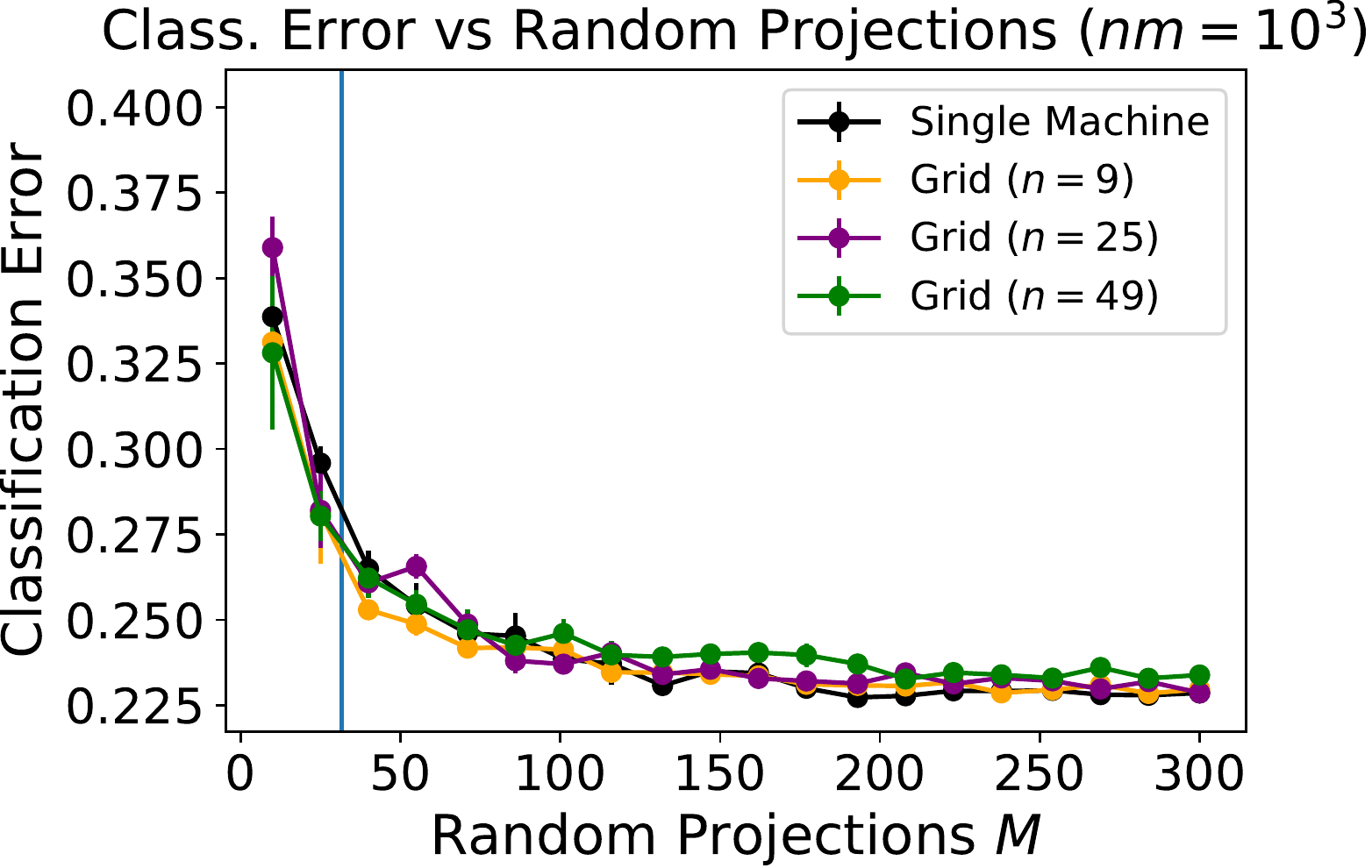}
    \caption{ Classification Error (if $y$ and $\widehat{y}$ are the true and predicted response respectively, error calculated is 0-1 loss) against number of Random Features $M$, with total sample size and maximum number of iterations $ t = nm = 10^{3}$. Vertical line in plots indicates $\sqrt{nm}$. 
    \textit{Left}: Cycle topology, \textit{Right}: Grid Topology. 
    }
    \label{fig:Plots1}
\end{figure}

Our theory predicts that the sub-optimality of more poorly connected networks decreases as the number of samples held by each agent increases. To investigate this, we repeat the above experiment for cycles and grids of sizes $n=25,49,100$ while varying the dataset size. Looking to Figure \ref{fig:Plots2}, we see that approximately $nm \approx 10^{3}$ samples are sufficient for a cycle topology of size $n=49$ to align with a single machine, meanwhile $10^{4}$ samples are required for a larger $n=100$ cycle. For a grid we see a similar phenomena, although with fewer samples required due to being better connected topology.
\begin{figure}[!h]
    \centering
    \includegraphics[width=0.35\textwidth]{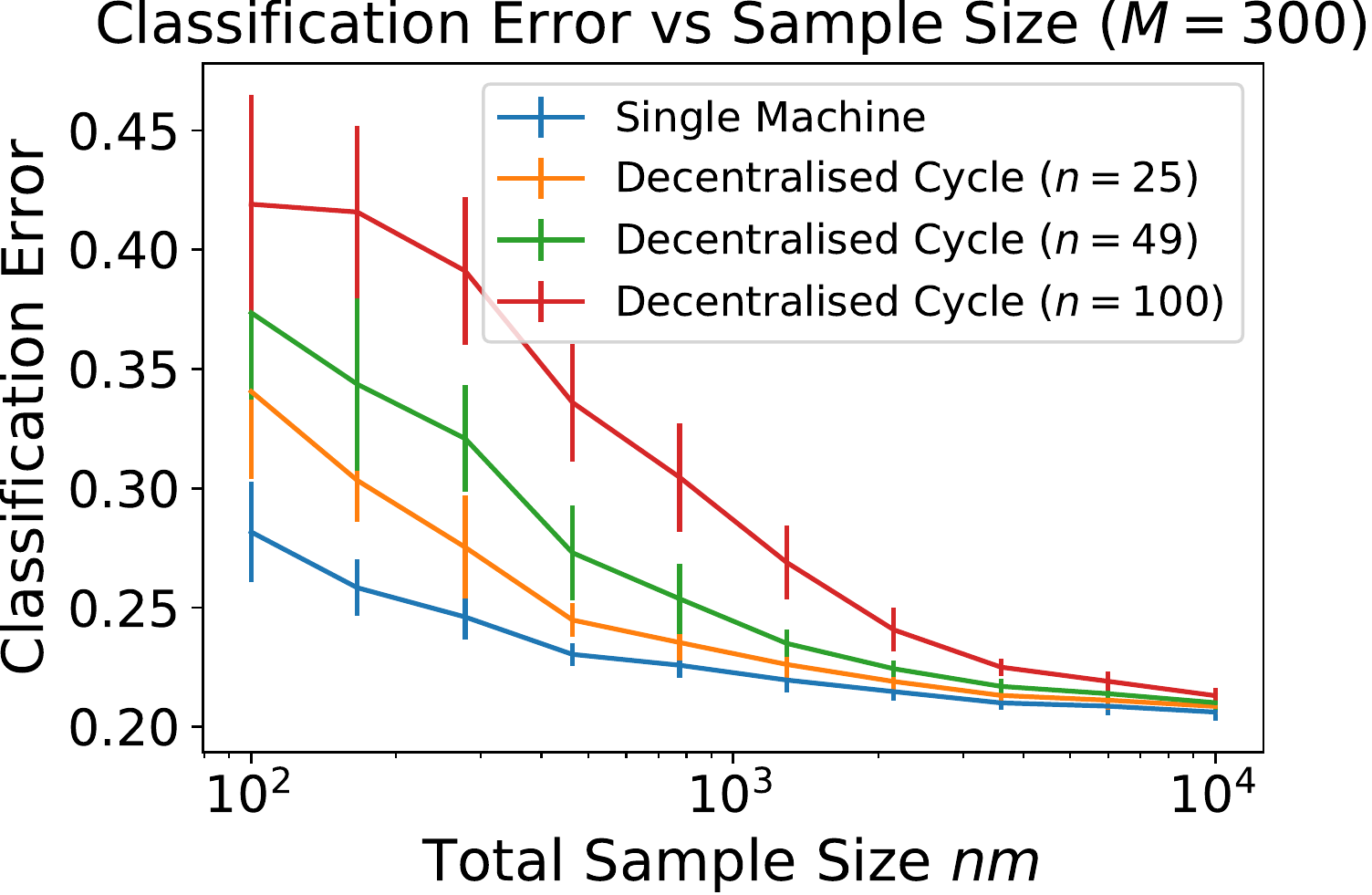}
    \includegraphics[width=0.35\textwidth]{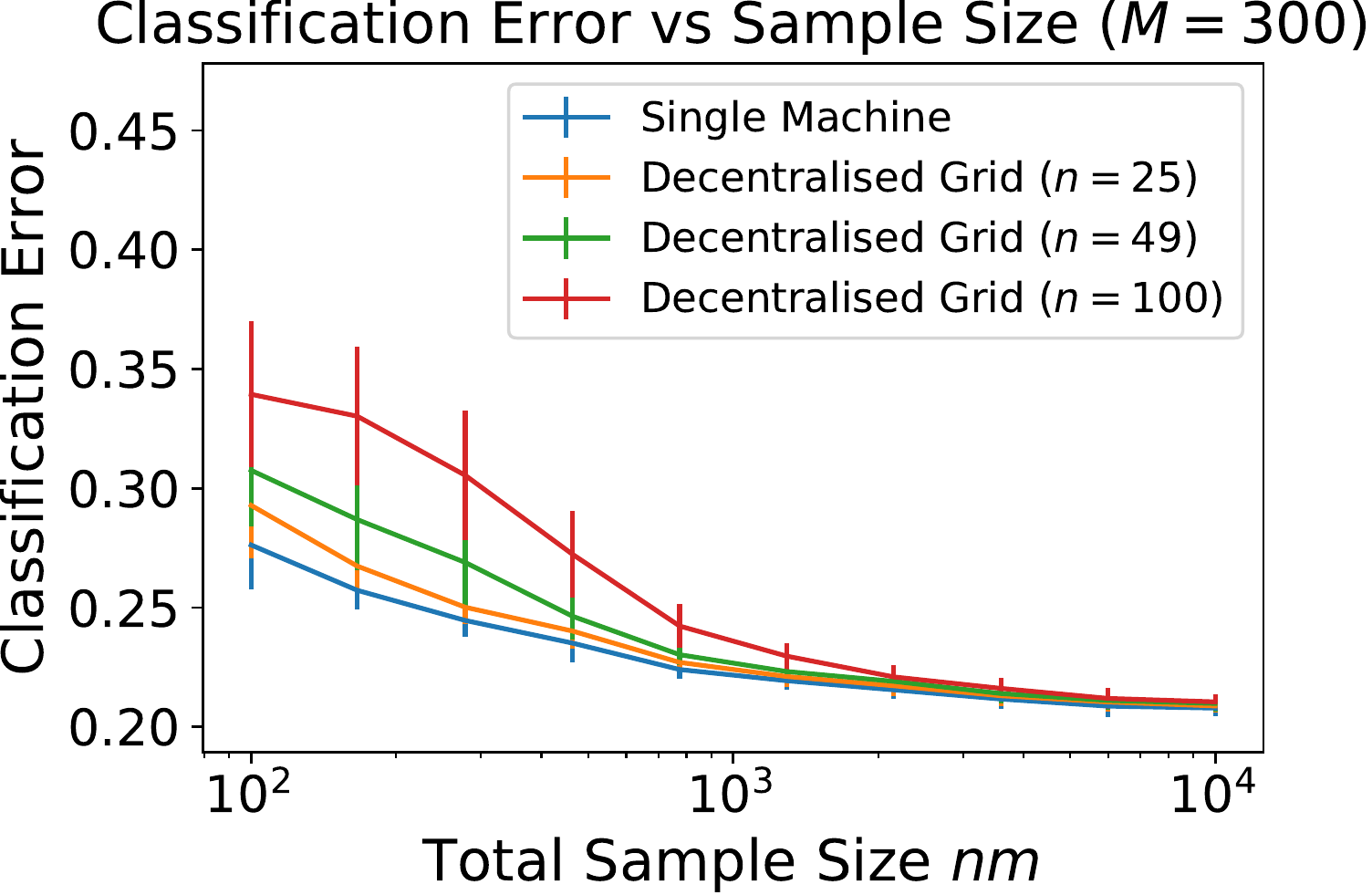}
    \caption{ Plots of Classification Error (computed as in Figure \ref{fig:Plots1}) against total number of samples $nm$, with $M=300$. Run for at most $t=10^{4}$ iterations, each point is an average of 20 sub-subsets of the SUSY, which Distributed Gradient Descent with Random Features is run on 5 times. 
    }
    \label{fig:Plots2}
\end{figure}

Our theory predicts that, given sufficiently many samples, the number of iterations for any network topology scales as those of single machine Gradient Descent. We look to Figure \ref{fig:Plots3} where the number of iterations required to achieve the minimum classification error (optimal stopping time) is plotted against the sample size. Observe that beyond approximately $10^{3}$ samples both grid and cycles of sizes  $n=25,49,100$ have iterates that scale at the same order as a single machine. Note that the number of iterations required by both topologies initially decreases with the sample size up to $10^{3}$. While not supported by our theory with constant step size, this suggests quantities held by agents become similar as agents hold more data, reducing the number of iterations to propagate information around the network. Investigation into this observation we leave to future work. 
\begin{figure}[!h]
    \centering
    \includegraphics[width=0.35\textwidth]{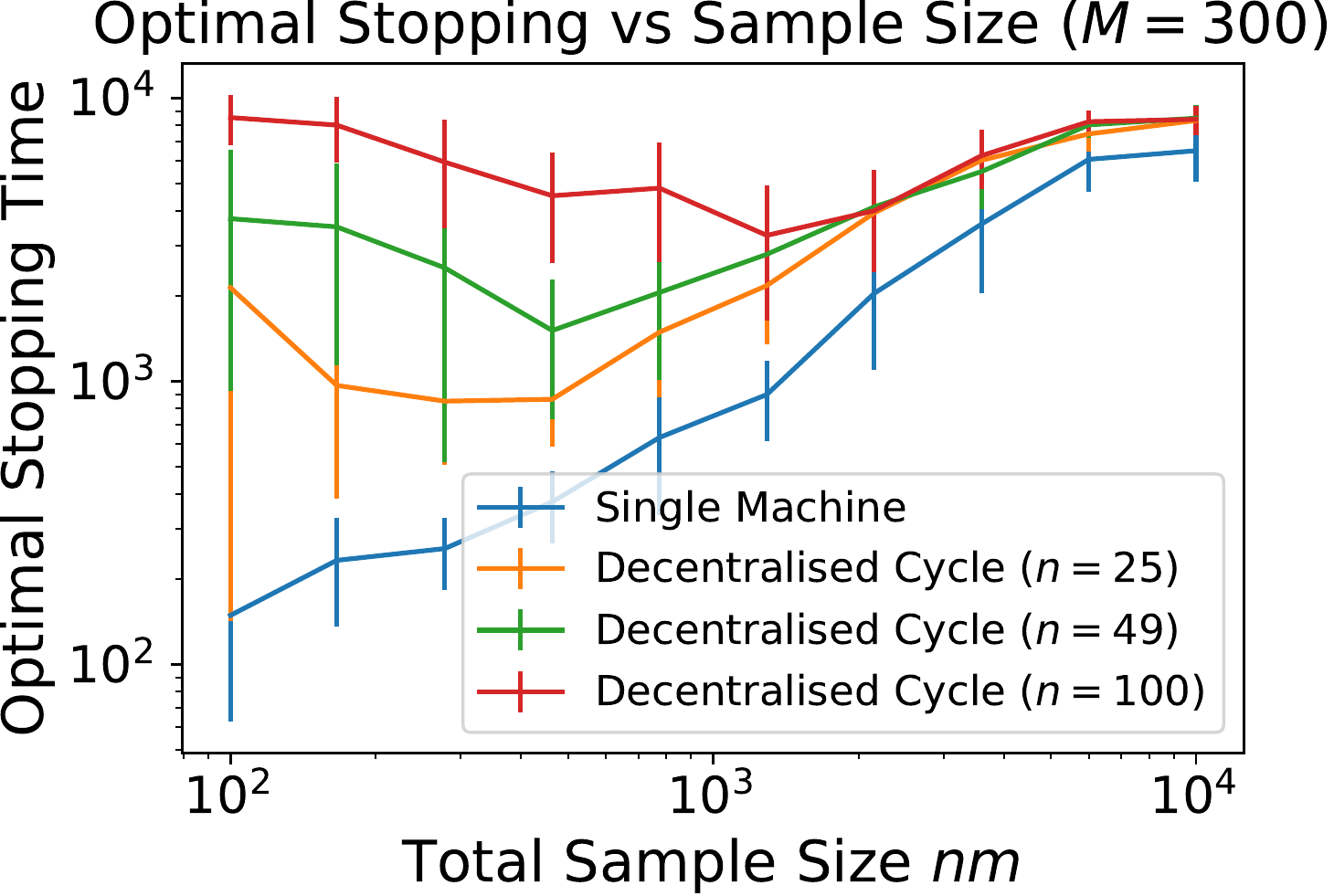}
    \includegraphics[width=0.35\textwidth]{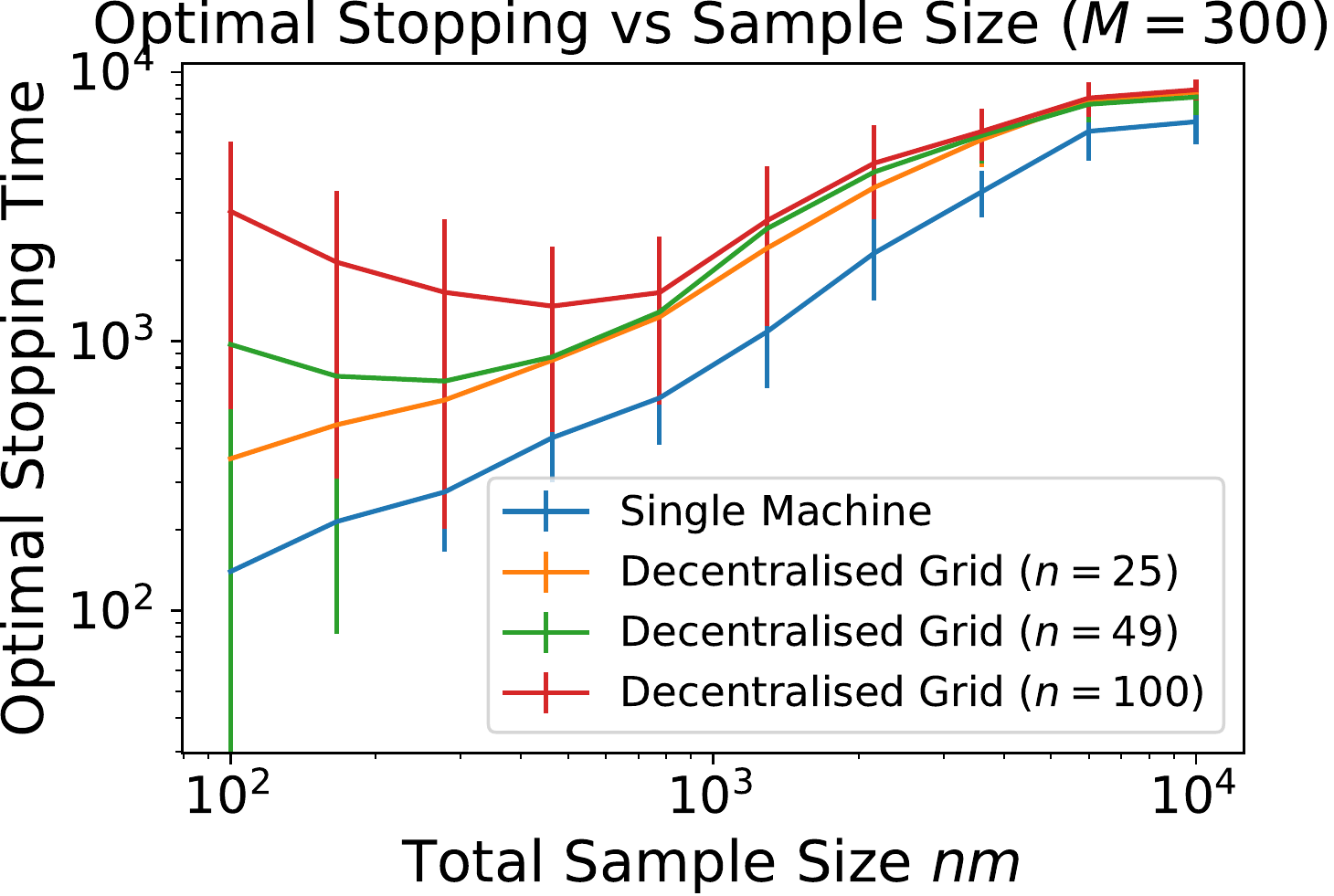}
    \caption{ Optimal Stopping Time (Number of iterations required) against sample size $nm$ ($\log-\log$ axis), with $M=300$. \textit{Left}: Cycle Topology, \textit{Right}: Grid topology. Each point is averaged over 20 sub-subsets of the SUSY. Distributed Gradient Descent with Random Features was repeated 5 times, with at most $10^{4}$ iterations.  
    }
    \label{fig:Plots3}
\end{figure}

\section{Conclusion}
\label{sec:Conclusion}
In this work we considered the performance of Distributed Gradient Descent with Random Features on the Generalisation Error, this being different from previous works which focused on training loss. Our analysis allowed us to understand the role of different parameters on the Generalisation error, and, when agents have sufficiently many samples with respect to the network size,  achieve a linear speed-up in runtime time for any network topology. 

Moving forward, it would be natural to extend our analysis to stochastic gradients \cite{lin2017optimal} or stochastic communication at each iteration \cite{Gossip}.

\section*{Acknowledgements}
D.R.~is supported by the EPSRC and MRC through the OxWaSP CDT programme
(EP/L016710/1).
Part of this work has been carried out at the Machine Learning Genoa (MaLGa) center, Università di Genova (IT).
L.R.~acknowledges the financial support of the European Research Council (grant SLING 819789), the AFOSR projects FA9550-17-1-0390  and BAA-AFRL-AFOSR-2016-0007 (European Office of Aerospace Research and Development), and the EU H2020-MSCA-RISE project NoMADS - DLV-777826.

\bibliographystyle{plain}

\bibliography{References}

\newpage 
\appendix

\section{Remarks}
\label{sec:Remarks}
In this section we give a number of remarks relating to content within the main body of the paper.

\begin{remark}[Sketching and Communication Savings]
 \label{remark:communication}
 We highlight that the Random Feature framework considered also incorporates a number of sketching techniques. For instance, when $\psi(x,\omega) = x^{\top} \omega$ where $\omega \sim \mathcal{N}(0,I)$ and the associated kernel is simply linear as \\ 
 $\E[\psi(x,\omega)\psi(x^{\prime},\omega)] \E[x^{\top} \omega \omega^{\top} x] = x^{\top} \E[ \omega \omega^{\top}] x = x^{\top} x^{\prime}$. The case $M < D$ then represents a simple setting in which communication savings can be achieved, as agents in this case would only need to communicate an $M$ dimensional vector instead of $D$.  A natural future direction would be to investigate whether there exists particular sketches/Random Features tailored to the objective of communication savings, in a similar manner to Orthogonal Random Features \cite{yu2016orthogonal}, Fast Food \cite{le2013fastfood} or Low-precision Random Features \cite{zhang2019low}.  Although, as noted in \cite{carratino2018learning},  some of these methods sample the features in a correlated manner, and thus, do not fit within the assumptions of this work. 
 \end{remark}

\begin{remark}[Previous Literature Decentralised Kernel Methods]
\label{remark:PreviousLit}
This remark highlights two previous works for Decentralised Kernel Methods. 
The work \cite{forero2010consensus} considers decentralised Support Vector Machines with potentially high-dimensional finite feature spaces that could approximate a non-linear kernel. They develop a variant of the Alternating Direction Method of Multiplers (ADMM) to target the augmented optimisation problem. In this case, the high-dimensional constraints across the agents are approximated  so the agents local estimated functions are equal on a subset of chosen points. Meanwhile \cite{koppel2018decentralized} consider online stochastic optimisation with penalisation between neighbouring agents. The penalisation introduced is an expectation with respect to a newly sampled data point and not in the norm of the Reproducing Kernel Hilbert Space. In both of these cases, the original optimisation problem is altered to facilitate a decentralised algorithm, but no guarantee is given on how these approximation impact statistical performance. 
\end{remark}
\begin{remark}[Concurrent Work]
\label{remark:COKE}
The concurrent work \cite{xu2020coke} consider the homogeneous setting where a network of agents have data from the same distribution and wish to learn a function within a RKHS that performs well on unseen data. The consensus optimisation formulation of the single machine explicitly penalised kernel learning problem is considered, and the challenges of decentralised kernel learning (as described in Section \ref{sec:DecentralisedKR} in the main body of the manuscript) are overcome by utilising Random Fourier Features. An ADMM method is developed to solve the consensus optimisation problem, and, provided hyper-parameters are tuned appropriately, optimisation guarantees are given. Due to considering the consensus optimisation formulation of a single machine penalised problem, the Generalisation Error is decoupled from the Optimisation Error. Therefore, while optimisation results for ADMM applied to consensus optimisation objectives \cite{shi2014linear} are applied, the statistical setting is not leveraged to achieve speed-ups.  It is then not clear how the network connectivity, number of samples held by agents and finer statistical assumptions (source and capacity) impacts either generalisation or optimisation performance. This is in contrast to our work, where we directly study the Generalisation Error of Distributed Gradient Descent with Implicit Regularisation, and show how the number of samples held by agents, network topology, step size and number of iterations can impact Generalisation Error. 
\end{remark}

\section{Analysis Setup}
This section provides the setup for the analysis. We adopt the notation of \cite{carratino2018learning}, which is included here for completeness. Section \ref{sec:AuxSequences} introduces additional auxiliary quantities required for the analysis. Section \ref{app:Notation} introduces notation for the operators required for the analysis. Section \ref{app:ErrorDecomp} introduces the error decomposition.

\subsection{Additional Auxiliary Sequences}
\label{sec:AuxSequences}
We begin by introducing some auxiliary sequences that will be useful in the analysis. Begin by defining $\{v_{t}\}_{t \geq 1}$ initialised at $v_{1} = 0$ and updated for $t \geq 1$
and updated 
\begin{align*}
    & v_{t+1} = v_t - 
    \eta 
    \int_{X} \big( \langle v_{t},\phi_{M}(x)\rangle - f_{\mathcal{H}}(x)\big) \phi_{M}(x) d \rho_{X}(x)
\end{align*}
Further for $\lambda > 0$ let 
\begin{align*}
    & \widetilde{u}_{\lambda} = \argmin_{u \in \mathbb{R}^{M}} \int_{X} 
    \big( \langle u , \phi_{M}(x) \rangle - f_{\mathcal{H}}(x)\big)^2 d \rho_{X}(x) + \lambda \|u\|^2,\\
    & u_{\lambda} = \argmin_{ u \in \mathcal{F} } 
    \int_{X} \big( \langle u , \phi(x) \rangle - y)^2 d \rho(x,y) + \lambda \|u\|^2, 
\end{align*}
where $(\mathcal{F},\phi)$ are feature space and feature map associated to the kernel $k$. As described previously, it will be useful to work with functions in  $L^2(X,\rho_{X})$, therefore define the functions 
\begin{align*}
    g_{t}  = \langle v_{t}, \phi_{M}(\cdot) \rangle, 
    \quad 
    \widetilde{g}_{\lambda} = \langle \widetilde{u}_{\lambda},\phi_{M}(\cdot)\rangle,
     \quad 
    g_{\lambda} = \langle u_{\lambda},\phi(\cdot)\rangle.
\end{align*}
The quantities introduced here in this section will be useful in analysing the \textit{Statistical Error} term.

\subsection{Notation}
\label{app:Notation}
Let $\mathcal{F}$ be the feature space corresponding to the kernel $k$ given by Assumption \ref{ass:FeatureRegularity}. 

Given $\phi: X \rightarrow \mathcal{F}$ (feature map), we define the operator $S: \mathcal{F} \rightarrow L^2(X,\rho_{X})$ as 
\begin{align*}
    (S \omega)(\cdot)  = \langle \omega, \phi(\cdot) \rangle_{\mathcal{F}},\quad 
    \forall \omega \in \mathcal{F}.
\end{align*}
If $S^{\star}$ is the adjoint operator of $S$, we let $C: \mathcal{F} \rightarrow \mathcal{F}$ be the linear operator $C = S^{\star}S$, which can be written as 
\begin{align*}
    C = \int_{X} \phi(x) \otimes \phi(x) d \rho_{X}(x).
\end{align*}
We also define the linear operator $L: L^2(X,\rho_{X}) \rightarrow L^2(X,\rho_{X})$ such that $L = S S^{\star}$, that can be represented as 
\begin{align*}
    (Lf)(\cdot) = \int_{X} \langle \phi(x),\phi(\cdot) \rangle_{\mathcal{F}} f(x) d \rho_{X}(x), 
    \quad 
    \forall f \in L^2(X,\rho_{X}). 
\end{align*}

We now define the analog of the previous operators where we use the feature map $\phi_{M}$ instead of $\phi$. We have $S_{M}:\mathbb{R}^{M} \rightarrow L^2(X,\rho_{X})$ defined as
\begin{align*}
    (S_{M} v)(\cdot) = \langle v,\phi_{M}(\cdot) \rangle_{\mathbb{R}^{M}}, \quad 
    \forall v \in \mathbb{R}^{M}
\end{align*}
together with $C_{M}: \mathbb{R}^{M} \rightarrow \mathbb{R}^{M}$and $L_{M}: L^2(X,\rho_{X}) \rightarrow L^2(X,\rho_{X})$ defined as $C_{M} = S^{\star}_{M} S_{M}$ and $L_{M} = S_{M} S^{\star}_{M}$ respectively. For $v \in \mathbb{R}^{M}$ note we have the equality 
\begin{align}
    \|S_M v \|_{\rho}^2 & = \int_{X} \langle v , \phi_{M}(x) \rangle^2 d \rho_{X}(x) \nonumber \\
    & = \int_{X} v^{\top} \phi_{M}(x) \otimes \phi_{M}(x) v  d \rho_{X}(x) \nonumber \\
    & = v^{\top} C_{M} v \nonumber \\
    & = \|C_{M}^{1/2} v\|^2 \label{equ:Isometry}
\end{align}
where we have denoted the standard Euclidean norm as $\|\cdot\|$.
Define the empirical counterpart of the previous operators for each agent. For each agent $v \in V$ define the operator $\widehat{S}_{M}^{(v)}: \mathbb{R}^{M} \rightarrow \mathbb{R}^{m}$ as 
\begin{align*}
    \widehat{S}^{(v) \top}_{M}  = \frac{1}{\sqrt{m}}(\phi_{M}(x_{1,v}),\dots,\phi_{M}(x_{m,v})),
\end{align*}
and with $\widehat{C}_{M}^{(v)}:\mathbb{R}^{M}\rightarrow \mathbb{R}^{M}$ and $\widehat{L}_{M}^{(v)}:\mathbb{R}^{m} \rightarrow \mathbb{R}^{m}$ are defined as $\widehat{C}_{M}^{(v)} = \widehat{S}^{(v) \top}_{M} \widehat{S}^{(v)}_{M}$ and $\widehat{L}_{M}^{(v)} = \widehat{S}^{(v)}_{M} \widehat{S}^{(v) \top}_{M}$ respectively. Moreover, define the empirical operators associated to all of the samples held by agents in the network. To do so index the agents in $V$ between $1$ and $n$, so $x_{i,j}$ is the $i$th data point held by agent $j$. Then, define the operator $\widehat{S}_{M}: \mathbb{R}^{M} \rightarrow \mathbb{R}^{nm}$ as 
\begin{align*}
    \widehat{S}^{\top}_{M} 
    & = \frac{1}{\sqrt{nm}}(\phi_M(x_{1,1}),\dots,\phi_{M}(x_{m,1}),\phi_{M}(x_{1,2}),\dots,\phi_{M}(x_{m,2}),\dots,\phi_{M}(x_{1,n}),\dots,\phi_{M}(x_{m,n}))\\
    & = 
    \frac{1}{\sqrt{n}}( \widehat{S}^{(1) \top}_{M},\dots,\widehat{S}^{(n) \top}_{M})
\end{align*}
and with $\widehat{C}_{M}:\mathbb{R}^{M}\rightarrow \mathbb{R}^{M}$ and $\widehat{L}_{M}:\mathbb{R}^{nm} \rightarrow \mathbb{R}^{nm}$ are defined as $\widehat{C}_{M} = \widehat{S}^{\top}_{M} \widehat{S}_{M}$ and $\widehat{L}_{M} = \widehat{S}_{M} \widehat{S}^{\top}_{M}$ respectively. From the above it is clear that we have $\widehat{C}_{M} = \frac{1}{n} \sum_{w \in V} \widehat{S}^{(w) \top}_{M}\widehat{S}^{(w)}_{M}= \frac{1}{n} \sum_{w \in V} C^{(w)}_{M} $. For some number $\lambda > 0 $  we let the operator plus the identity times $\lambda$ be denoted $L_{\lambda} = L + \lambda I$, and similarly for $\widehat{L}_{\lambda} $, as well as  $C_{M,\lambda} = C_{M} + \lambda I$ and $\widehat{C}_{M,\lambda}$.

\begin{remark}
Let $P: L^2(X,\rho_{X}) \rightarrow  L^2(X,\rho_{X})$ be the projection operator whose range is the closure of the range of $L$. Let $f_{\rho}:X \rightarrow \mathbb{R}$ be defined as 
\begin{align*}
    f_{\rho}(x)  = \int y d\rho(y|x).
\end{align*}
If there exists $f_{\mathcal{H}} \in \mathcal{H}$ such that 
\begin{align*}
    \inf_{f \in \mathcal{H}} \mathcal{E}(f) = \mathcal{E}(f_{\mathcal{H}})
\end{align*}
then 
\begin{align*}
    P f_{\rho} = S f_{\mathcal{H}}.
\end{align*}
or equivalently, there exists $g \in L^2(X,\rho_{X})$ such that 
\begin{align*}
    P f_{\rho} = L^{1/2} g.
\end{align*}
In particular, we have $R := \|f_{\mathcal{H}}\|_{\mathcal{H}} = \|g\|_{L^2(X,\rho_{X})}$. The above condition is commonly relaxed in approximation theory as 
\begin{align*}
    P f_{\rho} = L^r g
\end{align*}
with $1/2 \leq r \leq 1$. 
\end{remark}

With the operators introduced above and the above remark, we can rewrite the auxiliary objects respectively as 
\begin{align*}
    & \widehat{v}_{1} = 0; 
    \quad
    \widehat{v}_{t+1} = (I - \eta \widehat{C}_{M}) \widehat{v}_{t} + \eta \widehat{S}_{M}^{\top} \widehat{y} \\
    & \widetilde{v}_{1} = 0; 
    \quad 
    \widetilde{v}_{t+1} = (I - \eta C_{M})\widetilde{v}_{t} 
    + 
    \eta S^{\star}_{M} f_{\rho} \\
    & v_1 = 0 ; 
    \quad
    v_{t+1} = (I - \eta C_{M}) v_{t} + 
    \eta S^{\star}_{M} P f_{\rho} 
\end{align*}
where the vector of all $nm$ responses are $\widehat{y}^{\top} = (nm)^{-1/2}(y_{1,1},\dots,y_{1,m},y_{2,m},\dots,y_{n,m}) = (n)^{-1/2} (\widehat{y}_{1},\dots,\widehat{y}_{n})$, and each agents responses are, for $i = 1,\dots,n$, denoted $\widehat{y}_{v} = (m)^{-1/2}(y_{i,1}.\dots,y_{i,m})$. We then denote 
\begin{align*}
    & \widetilde{u}_{\lambda} = S^{\star}_{M}L^{-1}_{M,\lambda} P f_{\rho} \\
    & u_{\lambda}  = S^{\star} L_{\lambda}^{-1} P f_{\rho}. 
\end{align*}
Inductively the three sequences can be written as 
\begin{align*}
    & \widehat{v}_{t+1}  = \sum_{k=1}^{t} \eta (I - \eta \widehat{C}_{M})^{t-k} \widehat{S}_{M}^{\top} \widehat{y} \\
    & \widetilde{v}_{t+1} = \sum_{k=1}^{t} \eta (I - \eta C_{M})^{t-k} S^{\star}_{M} f_{\rho} \\
    & v_{t+1} =  \sum_{k=1}^{t} \eta (I - \eta C_{M})^{t-k} S^{\star}_{M} P f_{\rho} 
\end{align*}
\subsection{Error Decomposition}
\label{app:ErrorDecomp}
We can now write the deviation $\widehat{f}_{t+1,v} - f_{\mathcal{H}}$ using the operators 
\begin{align}
\label{equ:ErrorDecomp}
    \widehat{f}_{t+1,v} - f_{\mathcal{H}}
    & = 
    \underbrace{ S_{M} \widehat{\omega}_{t+1,v} - S_{M} \widehat{v}_{t}}_{\text{Network Error}}
    + 
    \underbrace{ S_{M} \widehat{v}_{t} - P f_{\rho} }_{\text{Statistical Error}}
\end{align}
where the first term aligns with the network error and the second with the statistical error. Each of these will be analysed in it own section. 

\section{Statistical Error}

In this section we summarise the analysis for the Statistical Error which has been conducted within \cite{carratino2018learning}. Here we provided the proof for completeness. Firstly, we further decompose the statistical error into the following terms 
\begin{align}
\label{equ:StatErrorDecomp}
    \|S_{M} \widehat{v}_{t+1} - P f_{\rho}\|_{\rho} 
    \leq
    & \underbrace{ 
    \| S_{M} \widehat{v}_{t+1} - S_{M}\widetilde{v}_{t+1} 
    +
    S_{M}\widetilde{v}_{t+1} - S_{M} v_t \|_{\rho} }_{\text{Sample Error}}  
    +
    \underbrace{ \| S_{M} v_{t+1} - L_{M} L_{M,\lambda}^{-1} P f_{\rho}\|_{\rho}  }_{\text{Gradient Descent and Ridge Regression}}\\
    & +\quad 
    \underbrace{ \| L_{M}L_{M,\lambda}^{-1} P f_{\rho} - L L^{-1}_{\lambda}P f_{\rho} \|_{\rho} }_{\text{Random Features Error}} 
    +
    \underbrace{ \| L L_{\lambda}^{-1}P f_{\rho} - P f_{\rho}\|_{\rho} }_{\text{Bias}}
    \nonumber
\end{align}
Each of the terms have been labelled to help clarity. The first term, \textit{sample error} includes the difference between the empirical iterations with sampled data $\widehat{v}_{t}$, as well as iterates under the population measure $v_t$. The second term \textit{Gradient Descent and Ridge Regression}  is the difference between the population variants of the Gradient Descent $v_t$ and ridge regression  $L_{M} L_{M,\lambda}^{-1} P f_{\rho}$ solutions. The third term \textit{Random Feature Error} accounts for the error introduced from using Random Features. Finally the \textit{Bias} term accounts for the bias introduced due to the regularisation. Each of these terms will be bounded within their own sub-section, except the \textit{Bias} term which will be bounded when bounds for all of the terms are brought together. 

The remainder of this section is then as follows. 
Section \ref{app:SampleError}, \ref{app:GradDescRR} and \ref{app:RandomFeatureError} give the analysis for the Sample Error, Gradient Descent and Ridge Regression and Random Feature Error error respectively. Section \ref{app:CombinedStatError} bounds the Bias and combines bounds for the previous terms.

\subsection{Sample Error}
\label{app:SampleError}
The bound for this term is summarised within the following Lemma which itself comes from Lemma 1 and 6 in \cite{carratino2018learning}.
\begin{lemma}[Sample Error]
\label{lem:SampleError}
Under assumptions \ref{ass:FeatureRegularity}, \ref{ass:moment} and \ref{ass:RKHS} , let $\delta \in (0,1)$, $\eta \in (0,\kappa^{-2})$. When 
\begin{align*}
    M \geq \big( 4 + 18 \eta t \big) \log \frac{12 \eta t}{\delta} 
\end{align*}
for all $t \geq 1$ with probability atleast $1-3\delta$ 
\begin{align*}
    & \| S_{M} \widehat{v}_{t} - S_{M}\widetilde{v}_{t} 
    +
    S_{M}\widetilde{v}_{t} - S_{M} v_t \|_{\rho}
    \leq 4 \Big( R \kappa^{2r} \Big( 1 + \sqrt{\frac{9}{M} \log \frac{M}{\delta}} \big( \sqrt{ \eta t} \vee 1 \big) \Big) 
    + 
    \sqrt{B} \Big)   \\
    & \quad \times\big(  12 + 4 \log (t) + \sqrt{ 2} \eta \big) 
    \Big( \frac{\sqrt{\eta t}}{nm} + \frac{\sqrt{ 2  \sqrt{p} q_0 \mathcal{N}(\frac{\kappa^2}{\eta t }) }}{\sqrt{nm}}\Big) 
    \log \frac{4}{\delta}
\end{align*}
where $q_0 = \max \big(2.55, \frac{2 \kappa^2}{\|L\|} \big)$
\end{lemma}
\begin{proof}
Apply Lemma 1 in \cite{carratino2018learning} to say $\| S_{M}\widetilde{v}_{t} - S_{M} v_t \|_{\rho} = 0$, meanwhile Lemma 6 in the same work to bound $\| S_{M} \widehat{v}_{t} - S_{M}\widetilde{v}_{t} \|$ with $\theta = 0$ and $T = t$.
\end{proof}

\subsection{Gradient Descent and Ridge Regression}
\label{app:GradDescRR}
This term is controlled by Lemma 9 in \cite{carratino2018learning}.
\begin{lemma}[Gradient Descent and Ridge Regression]
\label{lem:GDRR}
Under Assumption \ref{ass:RKHS} the following holds with probability $1-\delta$ for $\lambda = \frac{1}{\eta t}$ 
for $t \geq 1$
\begin{align*}
    \| S_{M} v_{t+1} - L_{M} L_{M,\lambda}^{-1} P f_{\rho}\|_{\rho} 
    & \leq 8 R \kappa^{2r}\Big( \frac{\log \frac{2}{\delta}}{M^r} 
    + 
    \sqrt{ \frac{ \mathcal{N}\big(\frac{1}{\eta t} \big)^{2r - 1} \log \frac{2}{\delta} }{M (\eta t)^{2r-1}}} \Big) 
    \log^{1-r}\big( 11 \kappa^2 \eta t\big) + \frac{2R}{(\eta t)^{r}}
\end{align*}
when 
\begin{align*}
    M \geq ( 4 + 18 \eta t ) \log \big( \frac{8 \kappa^2 \eta t }{\delta} \big)
\end{align*}
\end{lemma}

\subsection{Random Features Error}
\label{app:RandomFeatureError}
The following Lemma is from Lemma 8 of \cite{rudi2017generalization,carratino2018learning}.
\begin{lemma}
\label{lem:RFFError}
Under assumption \ref{ass:FeatureRegularity} and \ref{ass:RKHS} for any $\lambda > 0$, $\delta \in (0,1/2]$, when 
\begin{align*}
    M \geq \big( 4 + \frac{18 \kappa^2}{\lambda}\big) \log \frac{8 \kappa^2}{\lambda \delta}
\end{align*}
the following holds with probability at least $1-2 \delta$
\begin{align*}
    \| L_{M}L_{M,\lambda}^{-1} P f_{\rho} - L L^{-1}_{\lambda}P f_{\rho} \|_{\rho}
    \leq 
    4 R \kappa^{2r} \Bigg( \frac{\log \frac{2}{\delta}}{M^r} + \sqrt{\frac{\lambda^{2r-1} \mathcal{N}(\lambda)^{2r - 1} \log \frac{2}{\delta}}{M}} \Bigg) q^{1-r}
\end{align*}
where $q = \log \frac{11 \kappa^2}{\lambda}$
\end{lemma}

\subsection{Combined Error Bound}
\label{app:CombinedStatError}
The following Lemma combines the error bounds. 
\begin{lemma}
\label{lem:StatisticalError}
Under assumption \ref{ass:RandomFeatures} to \ref{ass:moment}, let $\delta \in (0,1)$ and $\eta \in (0,\kappa^{-2})$ when 
\begin{align*}
    M \geq (4 + 18 \eta t \kappa^2 ) \log \frac{60 \kappa^2 \eta t }{\delta}
\end{align*}
the following holds with probability greater than $1-\delta$
\begin{align*}
    & \|S_{M}\widehat{v}_{t+1}  - P f_{\rho}\|_{\rho}^2 
     \leq 
    c_1^2 \Big( 1 \vee \frac{ (\eta t \vee 1 ) \log \frac{3 M}{\delta}}{M} \Big) \Big(  \frac{\eta t}{(nm)^2} \vee \frac{\mathcal{N}(\frac{1}{\eta t} )}{nm} \Big) \log^2(t) \log^2 \frac{12}{\delta}  \\
    & \quad 
     + c_2^2 
    \Big( \frac{1}{M^{2r}}  \vee \frac{ \mathcal{N}(\frac{1}{\eta t} )^{2r - 1} }{ M (\eta t )^{2r-1} } \Big)
    \log^{2(1-r)}( 11 \kappa^2 \eta t) 
    \log^2 \big( \frac{6}{ \delta}\big) 
    + \frac{c_3^2}{ (\eta t)^{2r} }
\end{align*}
where the constants 
\begin{align*}
    c_1 & = 8 \times 12 \times 15 \big( \sqrt{B} \vee (R \kappa^{2r}) \big) ( 1 \vee \sqrt{ 2 \sqrt{p} q_0}) \\
    c_2 & = 24 R \kappa^{2r}\\
    c_3 & = 3R
\end{align*}
\end{lemma}
\begin{proof}[Lemma \ref{lem:SampleError}]
Begin fixing $\lambda = \frac{1}{\eta t}$ and bounding the bias from Lemma 5 of \cite{rudi2017generalization} as 
\begin{align*}
    \| L L_{\lambda}^{-1}P f_{\rho} - P f_{\rho}\|_{\rho}
    \leq 
    R \lambda^{r}.
\end{align*}
Now use Lemma \ref{lem:SampleError} to bound the Sample Error, Lemma \ref{lem:GDRR} for the Gradient Descent and Ridge Regression Term, and \ref{lem:RFFError} for the Random Features Error. With a union bound, note that the conditions on $M$ for each of these Lemmas is satisfied by $M \geq (4 + 18 \eta t \kappa^2 ) \log \frac{60 \kappa^2 \eta t }{\delta}$. Cleaning up constants and squaring then yields the bound. 

\end{proof}

\section{Network Error}
\label{sec:NetworkErrorBounds}
In this section we the proof of the following bound on the network error, which improves upon \cite{richards2019optimal}. This section is then structured as follows. Section \ref{App:Network:ErrorDecomp} provides the error decomposition for the Network Error. Section \ref{app:PrelimLemma} introduces a number of prelimary lemmas utilised within the analysis. Section \ref{sec:E1}, \ref{sec:E2}, \ref{sec:E3}, \ref{sec:E4} and \ref{sec:E5} then provides bounds for each of the error terms that arise within the decomposition.

\subsection{Error Decomposition}
\label{App:Network:ErrorDecomp}
Recall the vector of observations associated to agent $v \in V$ is denoted $\widehat{y}_{v} = \frac{1}{\sqrt{m}}(y_{1,v},\dots,y_{m,v})$. 
Using the previously introduced notation note that we can write the Distributed Gradient Descent iterates as for $t \geq 1$ and $v \in V$
\begin{align*}
    \widehat{\omega}_{t+1,v} =
    \sum_{w \in V}
    P_{vw}
    \Big(
    \widehat{\omega}_{t,w}
    - 
    \eta 
    \widehat{C}^{(w)}_{M}
    \widehat{\omega}_{t,w}
    + 
    \eta \widehat{S}^{(w) \top}_{M} \widehat{y}_{w}
    \Big)
\end{align*}
Centering the iterates around the population sequence $\widetilde{v}_{t}$ we have from the doubly stochastic property of $P$
\begin{align*}
    \widehat{\omega}_{t+1,v} - \widetilde{v}_{t+1} 
    & =
    \sum_{w \in V}
    P_{vw}
    \Big(
    \widehat{\omega}_{t,w} - \widetilde{v}_{t}
    + \eta 
    \big\{
    ( C_{M} \widetilde{v}_{t} - S_{M}^{\star} f_{\rho} ) 
    - 
    (
    \widehat{C}^{(w)}_{M} \widehat{\omega}_{t,w}
    + 
    \widehat{S}^{(w) \top }_{M} \widehat{y}_{w}
    )
    \big\}
    \Big) \\
    & = 
    \sum_{w \in V}
    P_{vw}
    \Big(
    (I - \widehat{C}^{(w)}_{M}) 
    (\widehat{\omega}_{t,w} - \widetilde{v}_{t})
    + \eta 
    \underbrace{ 
    \big\{
    ( C_{M} \widetilde{v}_{t} - S_{M}^{\star} f_{\rho} ) 
    - 
    (
    \widehat{C}^{(w)}_{M} \widetilde{v}_{t}
    + 
    \widehat{S}^{(w) \top }_{M} \widehat{y}_{w}
    )
    \big\}
    }_{N_{t,w}}
    \Big) \\
    & = 
    \sum_{w \in V}
    P_{vw}
    \Big(
    (I - \widehat{C}^{(w)}_{M}) 
    (\widehat{\omega}_{t,w} - \widetilde{v}_{t})
    + \eta 
    N_{t,w}
    \Big)
\end{align*}
where we have defined the error term 
\begin{align*}
    N_{t,w} := 
    ( C_{M} \widetilde{v}_{t} - S_{M}^{\star} f_{\rho} ) 
    - 
    (
    \widehat{C}^{(w)}_{M} \widetilde{v}_{t}
    + 
    \widehat{S}^{(w) \star}_{M} \widehat{y}_{w}
    )
    \quad 
    \forall s \geq 1\, w \in V.
\end{align*}
Note that a similar set of calculation can be performed for the iterates $\widehat{v}_{t}$ leading to the recursion for $v \in V$ initialised at $\widehat{v}_{1,v} = 0$ and updated for $t \geq 1$
\begin{align*}
    \widehat{v}_{t+1,v} - \widetilde{v}_{t+1} = \sum_{w \in V} \frac{1}{n}
    \Big(
    (I - \widehat{C}^{(w)}_{M}) 
    (\widehat{v}_{t,w} - \widetilde{v}_{t})
    + \eta 
    N_{t,w}
    \Big) 
\end{align*}
For a path indexed from time step $t$ to $k$ such that $1 \leq k \leq t$ as $w_{t:k} = (w_t,w_{t-1},\dots,w_{k}) \in V^{t-k+1}$, let the product of operators be denoted 
\begin{align}
    \Pi(w_{t:k}) = (I - \widehat{C}^{(w_{t})}_{M})(I - \widehat{C}^{(w_{t-1})}_{M})\dots (I - \widehat{C}^{(w_{k})}_{M})
\end{align}
Meanwhile for $k > t$ we say $\Pi(w_{t:k}) = I$.  Unravelling the sequences $\widehat{\omega}_{t+1,v} - \widetilde{v}_{t+1} $ and $\widehat{v}_{t+1} - \widetilde{v}_{t+1}$ with the above notation and taking the difference we then have 
\begin{align*}
    \widehat{\omega}_{t+1,v} - \widehat{v}_{t+1} 
    & = 
    \sum_{k=1}^{t} \eta \sum_{w_{t:k} \in V^{t-k+1}} 
    \big(  P_{v w_{t:k}} - \frac{1}{n^{t-k+1}}  \big)
    \Pi(w_{t:k+1})
    N_{k,w_{k}} \\
    & = 
    \sum_{k=1}^{t} \eta \sum_{w_{t:k} \in V^{t-k+1}} 
    \Delta(w_{t:k})\Pi(w_{t:k+1})
    N_{k,w_{k}} 
\end{align*}
where we have introduced the notation where we have denoted $\big(  P_{v w_{t:k}} - \frac{1}{n^{t-k+1}}  \big) = \Delta(w_{t:k}) \in \mathbb{R}$. Introduce notation for the difference between the product of operators indexed by the paths and the population equivalent 
\begin{align*}
    \Pi^{\Delta}(w_{t:k+1}) : = \Pi(w_{t:k+1}) - (I - \eta C_{M})^{t-k}.
\end{align*}
Fixing some $t^{\star} \in \mathbb{N}$ and supposing that $t > 2t^{\star} \geq 2$, observe that we can then write, for $k \leq t-t^{\star} -1$,
\begin{align*}
    & \Pi^{\Delta}(w_{t:k+1}) \\
    & = 
    \Pi(w_{t:k+1}) - \Pi(w_{t:k+t^{\star} +1})(I - \eta C_{M})^{t^{\star}} + \Pi(w_{t:k+t^{\star} +1})(I - \eta C_{M})^{t^{\star}} - (I - \eta C_{M})^{t-k} \\
    & =
    \Pi(w_{t:k+t^{\star} +1})\Pi^{\Delta}(w_{k+t^{\star}:k+1}) 
    + 
    \Pi^{\Delta}(w_{t:k+t^{\star} +1})(I - \eta C_{M})^{t^{\star}}
\end{align*}
where we have replaced the first $t^{\star}$ operators in $\Pi(w_{t:k})$ with the population variant $(I - \eta C_{M})$. Plugging this in then yields 
\begin{align*}
    & \widehat{\omega}_{t+1,v} - \widehat{v}_{t+1} 
     = 
     \sum_{k=1}^{t} 
     \eta 
    \sum_{w_{t:k} \in V^{t-k+1}} 
    \Delta(w_{t:k})(I - \eta C_{M})^{t-k}
    N_{k,w_{k}}  \\
    & \quad + 
    \sum_{k=t-2t^{\star} }^{t} \eta \sum_{w_{t:k} \in V^{t-k+1}} 
    \Delta(w_{t:k})\Pi^{\Delta}(w_{t:k+1})
    N_{k,w_{k}}  \\
    & \quad + 
    \sum_{k= 1 }^{t-2t^{\star} - 1} \eta \sum_{w_{t:k} \in V^{t-k+1}} 
    \Delta(w_{t:k})\Pi(w_{t:k+t^{\star} +1 }) \Pi^{\Delta}(w_{k+t^{\star}: k +1}) 
    N_{k,w_{k}}  \\ 
    & \quad + 
    \sum_{k= 1 }^{t-2t^{\star} - 1}  \eta 
    \sum_{w_{t:k} \in V^{t-k+1}} 
    \Delta(w_{t:k})\Pi^{\Delta}(w_{t:k+t^{\star} +1}) (I - \eta C_M)^{t^{\star}}
    N_{k,w_{k}} 
\end{align*}
where we split the series off for paths shorter than $2t^{\star}$. Note for the first and last term above, elements in the series can be simplified by summing over the nodes in the path. Defining for $s \geq 1$ and $v,w \in V$ the difference $\Delta^s(v,w) = P^{s}_{vw} - \frac{1}{n} $, we get for the first term when $k < t$
\begin{align*}
     \sum_{w_{t:k} \in V^{t-k+1}} 
    \Delta(w_{t:k})(I - \eta C_{M})^{t-k}
    N_{k,w_{k}} & 
    = \sum_{w_{k} \in V} 
    \Big( \sum_{w_{t:k+1} \in V^{t-k} } \Delta(w_{t:k}) \Big) 
    (I - \eta C_{M})^{t-k}
    N_{k,w_{k}}\\ 
     & = 
    \sum_{w \in V} \Delta^{t-k}(v,w)
    (I - \eta C_{M})^{t-k}
    N_{k,w} 
\end{align*}
where $\sum_{w_{t:k+1} \in V^{t-k} } \Delta(w_{t:k}) = \sum_{w_{t:k+1} \in V^{t-k} } P_{v w_{t:k}} - \sum_{w_{t:k+1} \in V^{t-k} } \frac{1}{n^{t-k+1}} = P^{t-k}_{vw} - \frac{1}{n} = \Delta^{t-k}(v,w)$. 
Meanwhile for the last term we can sum over the last $t^{\star}$ nodes in the path $w_{t:k}$, that is with
\begin{align*}
    \sum_{w_{k+t^{\star}:k+1} \in V^{t^{\star}}}
    \Delta(w_{t:k}) &  = 
    \sum_{w_{k+t^{\star}:k+1} \in V^{t^{\star}}}
    P_{v w_{t:k}} - \frac{1}{n^{t-k+1}}\\
    & = 
    P_{v w_{t:k+t^{\star} +1}}
    \sum_{w_{k+t^{\star}:k+1} \in V^{t^{\star}}}
    P_{w_{k+t^{\star} +1 : k }} 
    - 
    \sum_{w_{k+t^{\star}:k+1} \in V^{t^{\star}}}
    \frac{1}{n^{t-k+1}} \\
    & = 
    P_{vw_{t:k+t^{\star}+1}}(P^{t^{\star}})_{w_{k+t^{\star} + 1} w_{k}} 
    - 
    \frac{1}{n^{t-t^{\star}-k +1}} \\
    & = 
    P_{vw_{t:k+t^{\star}+1}} 
    ( 
    (P^{t^{\star}})_{w_{k+t^{\star} + 1} w_{k}}   - \frac{1}{n})
    + 
    \frac{1}{n}
    (P_{vw_{t:k+t^{\star}+1}}  - \frac{1}{n^{t-k-t^{\star}}})\\
    & = 
    P_{vw_{t:k+t^{\star}+1}} 
    \Delta^{t^{\star}}(w_{k+t^{\star} +1},w_{k})
    + 
    \frac{1}{n}
    \Delta(w_{t:k+t^{\star} +1})
\end{align*}
Plugging this in we get for $1 \leq k \leq t-2t^{\star} - 1$
\begin{align*}
    & \sum_{w_{t:k} \in V^{t-k+1}} 
    \Delta(w_{t:k})\Pi^{\Delta}(w_{t:k+t^{\star} +1}) (I - \eta C_M)^{t^{\star}}
    N_{k,w_{k}}  \\
    & = 
    \sum_{w_{k} \in V} 
    \sum_{w_{t:k+t^{\star} +1} \in V^{t-t^{\star} -k}}
    \Big( 
    \sum_{w_{k+t^{\star}:k+1} \in V^{t^{\star}}}
    \Delta(w_{t:k})
    \Big)
    \Pi^{\Delta}(w_{t:k+t^{\star} +1}) (I - \eta C_M)^{t^{\star}} N_{k,w_{k}} \\
    & = 
    \sum_{w_{k} \in V} 
    \sum_{w_{t:k+t^{\star} +1} \in V^{t-t^{\star} -k}}
    P_{vw_{t:k+t^{\star}+1}} 
    \Delta^{t^{\star}}(w_{k+t^{\star} +1},w_{k})   \Pi^{\Delta}(w_{t:k+t^{\star} +1}) (I - \eta C_M)^{t^{\star}} N_{k,w_{k}}\\
    &\quad + 
    \frac{1}{n}
    \sum_{w_{k} \in V} 
    \sum_{w_{t:k+t^{\star} +1} \in V^{t-t^{\star} -k}}
    \Delta(w_{t:k+t^{\star} +1})
    \Pi^{\Delta}(w_{t:k+t^{\star} +1}) (I - \eta C_M)^{t^{\star}} N_{k,w_{k}}\\
    & 
    = 
    \sum_{w_{k} \in V} 
    \sum_{w_{t:k+t^{\star} +1} \in V^{t-t^{\star} -k}}
    P_{vw_{t:k+t^{\star}+1}}  \Delta^{t^{\star}}(w_{k+t^{\star} +1},w_{k}) \Pi^{\Delta}(w_{t:k+t^{\star}+1}) (I - \eta C_M )^{t^{\star}}
    N_{k,w_{k}} \\
    & \quad + 
    \sum_{w_{t:k+t^{\star} +1} \in V^{t-t^{\star} -k}}
    \Delta(w_{t:k+t^{\star} +1})
    \Pi^{\Delta}(w_{t:k+t^{\star}+1})
    (I - \eta C_M )^{t^{\star}}
    N_{k}
\end{align*}
where at the end for the second term we have
\begin{align*}
    \frac{1}{n} \sum_{w_{k} \in v} N_{k,w_{k}} = N_{k} = 
    ( C_{M} \widetilde{v}_{t} - \mathcal{S}_{M}^{\star} f_{\rho} ) 
    - 
    (
    \widehat{C}_{M} \widetilde{v}_{t}
    + 
    \widehat{S}^{\top}_{M} \widehat{y}
    )
    \quad 
    \forall k \geq 1 .
\end{align*}
Plugging the above in, using the isometry property \eqref{equ:Isometry} and triangle inequality we get 
\begin{align}
    & \|S_{M}(\widehat{\omega}_{t+1,v} - \widehat{v}_{t+1})\|_{\rho}
    \leq 
     \sum_{k=1}^{t} \eta \sum_{w \in V} |\Delta^{t-k}(v,w)| \|C_{M}^{1/2}(I - \eta C_{M})^{t-k}N_{k,w}\| 
     \nonumber 
     \\
    & \quad + 
    \sum_{k=t-2t^{\star} }^{t} \eta \sum_{w_{t:k} \in V^{t-k+1}} 
    |\Delta(w_{t:k})|
    \| C_{M}^{1/2}\Pi^{\Delta}(w_{t:k+1})
    N_{k,w_{k}}\| 
    \nonumber 
    \\
    & \quad + 
    \sum_{k= 1 }^{t-2t^{\star} - 1} \eta \sum_{w_{t:k} \in V^{t-k+1}} 
    |\Delta(w_{t:k})| 
    \|C_{M}^{1/2} \Pi(w_{t:k+t^{\star} +1 }) \Pi^{\Delta}(w_{k+t^{\star}: k +1}) 
    N_{k,w_{k}}\|  
    \nonumber 
    \\ 
    & \quad + 
    \sum_{k=1}^{t - 2t^{\star} -1} 
    \eta 
    \sum_{w_{k} \in V} 
    \sum_{w_{t:k+t^{\star} +1} \in V^{t-t^{\star} -k}}
    | P_{vw_{t:k+t^{\star}+1}}  \Delta^{t^{\star}}(w_{k+t^{\star} +1},w_{k})| \nonumber 
    \\
    & \quad\quad\quad\quad\quad\quad\quad\quad\quad\quad\quad\quad \times 
    \| 
    C_{M}^{1/2} 
    \Pi^{\Delta}(w_{t:k+t^{\star}+1}) (I - \eta C_{M} )^{t^{\star}}
    N_{k,w_{k}}\|  
    \nonumber 
    \\
    & \quad + 
    \sum_{k=1}^{t - 2t^{\star} -1} 
    \eta
    \Big\|
    \sum_{w_{t:k+t^{\star} +1} \in V^{t-t^{\star} -k}} \Delta(w_{t:k+t^{\star} +1}) C_{M}^{1/2} 
    \Pi^{\Delta}(w_{t:k+t^{\star}+1})
    (I - \eta C_{M} )^{t^{\star}}
    N_{k}\Big\| \nonumber 
    \\
    & = 
    \bf{E}_1 + \bf{E}_2 + \bf{E}_3 + \bf{E}_4 + \bf{E}_5
    \label{equ:NetworkErrorDecomp}
\end{align}
where we have respectively labelled the error terms $\bf{E}_i$ for $i=1,\dots,5$. We will aim to construct high probability bounds for each of these error terms within the following sections. This will rely on utilising the mixing properties of $P$ to control the deviations $\Delta^{s}(v,w)$  for some $s \geq 1$ and $v,w \in V$, the contractive property of operators $C_{M}^{1/2}(I - \eta C_{M})^{k}$ for some $k \in \mathbb{N}_{+}$ as well as concentration of the error terms $N_{k,w}$ and $N_{k}$ for $k \geq 1$ and $w \in V$. These are summarised within the following section. 

\subsection{Preliminary Lemmas}
\label{app:PrelimLemma}
In this section we provide some Lemmas that will be useful for later. We begin with the following that bounds the deviation  $\Delta^{s}(v,w)$ in terms of the second largest eigenvalue in absolute value of $P$. 
\begin{lemma}[Spectral Bound]
\label{lem:SpectralBound}
Let $s \geq 1$, $v \in V$. Then the following holds 
\begin{align*}
    \sum_{w \in V} |\Delta^{s}(v,w)| 
    \leq 
    2 (\sqrt{n} \sigma_2^{s} \wedge 1)
\end{align*}
\end{lemma}
\begin{proof}[Lemma \ref{lem:SpectralBound}]
Let $e_v \in \mathbb{R}^{n}$ denoting the standard basis with a 1 in the place associated to agent $v$. Observe that we can write the deviation in terms of the $\ell_1$ norm $\sum_{w \in V} |\Delta^{s}(v,w)| = \|e_v^{\top} P^{s} - \frac{1}{n} \mathbf{1}\|_{1}$. We immediately have an upper bound from triangle inequality that $\sum_{w \in V} |\Delta^{s}(v,w)| \leq \|\|e_v^{\top} P^{s}\|_{1} + \|\frac{1}{n} \mathbf{1}\|_{1} = 2 $. Meanwhile, we can also go to the $\ell_2$ norm and bound 
\begin{align*}
    \|e_v^{\top} P^{s} - \frac{1}{n} \mathbf{1}\|_{1} 
    \leq 
    \sqrt{n} \|e_v^{\top} P^{s} - \frac{1}{n} \mathbf{1}\|_{2} 
    \leq \sqrt{n} \sigma_2^{s}.
\end{align*}
The bound is arrived at by taking the maximum between the two upper bounds. 
\end{proof}
The following Lemma bonds the norm of contractions 
\begin{lemma}[Contraction]
\label{Lem:OperatorNorm}
Let $\mathcal{L}$ be a compact, positive operator on a separable Hilbert Space $H$. Assume that $\eta \|\mathcal{L}\| \leq 1$. For $t \in \mathbb{N}$, $a > 0$ and any non-negative integer $k \leq t-1$ we have 
\begin{align*}
\| (I - \eta \mathcal{L})^{t-k} \mathcal{L}^{a} \| \leq 
\bigg(  \frac{1}{ \eta(t-k)} \bigg)^{a}.
\end{align*}
\end{lemma}
\begin{proof}[Lemma \ref{Lem:OperatorNorm}]
The proof in Lemma 15 of \cite{lin2017optimal} considers this result with  $a=r$. 
The proof for more general $a > 0$ follows the same steps.
\end{proof}
The following remark will summarise how the above Lemma is applied to control series of contractions. 
\begin{remark}[Lemma \ref{Lem:OperatorNorm}]
\label{remark:Contraction}
Lemma \ref{Lem:OperatorNorm} will be applied to control series of the form $\eta \sum_{k=1}^{t} \| (I - \eta \mathcal{L})^{t-k} \mathcal{L}^{a} \| $ for some $t \geq 3$, most notably with powers $a=1,1/2$. In the case $a=1$ we immediately have the bound 
\begin{align*}
    \eta \sum_{k=1}^{t} \| (I - \eta \mathcal{L})^{t-k} \mathcal{L} \|
    & = 
    \eta \sum_{k=1}^{t-1} \| (I - \eta \mathcal{L})^{t-k} \mathcal{L} \|
    + 
    \eta \|\mathcal{L}\| \\
    & \leq 
    \eta \sum_{k=1}^{t-1} \frac{1}{\eta(t-k)}
    + 
    \eta \|\mathcal{L}\| \\
    & \leq
    5 \log(t)
\end{align*}
where we have bounded the series $\sum_{k=1}^{t-1} \frac{1}{t-k} \leq 4 \log(t)$ and $\eta \|\mathcal{L}\| \leq 1$. Similarly for $a =1/2$ we have 
\begin{align*}
    \eta \sum_{k=1}^{t} \| (I - \eta \mathcal{L})^{t-k} \mathcal{L}^{1/2} \|
    & \leq 
    \eta \sum_{k=1}^{t-1} \frac{1}{\sqrt{\eta(t-k)}}
    + 
    \eta \|\mathcal{L}^{1/2}\| \\
    & \leq 
    3 \sqrt{\eta t} + \sqrt{\eta} \\
    & \leq 5 \sqrt{\eta t}
\end{align*}
where we have bounded the series $\sum_{k=1}^{t-1} \frac{1}{\sqrt{t-k)}} \leq 4 \sqrt{t}$, see for instance Lemma 23 in \cite{richards2019optimal} with $q = 0$, as well as the bound that $\sqrt{\eta} \|L^{1/2} \| \leq 1$. 
\end{remark}

Now for $\lambda > 0$ define the effective dimension associated the feature map $\phi_{M}$, that is 
\begin{align*}
    \mathcal{N}_{M}(\lambda) := \trace\big( \big(L_{M} + \lambda I)^{-1} L_{M} \big).
\end{align*}
Given this, the following Lemma summarises the concentration results used within our analysis. 
\begin{lemma}[Concentration of Error]
\label{lem:conc}
Let $\delta \in (0,1]$, $n,m,M \in \mathbb{N}_{+}$, $\lambda > 0$ and $\eta\kappa^2 \leq 1$. Under assumption \ref{ass:FeatureRegularity},\ref{ass:RKHS} and \ref{ass:moment}  we have with probability greater than $1-\delta$ for $1 \leq  k  \leq t$
\begin{align*}
    & \max_{w \in V} \|C_{M,\lambda}^{-1/2} ( C_{M} - \widehat{C}_{M}^{(w)})\| \leq   
    2 \kappa \Big( \frac{2 \kappa}{m \sqrt{\lambda}} + \sqrt{\frac{ \mathcal{N}_{M}(\lambda)}{m}}\Big) \log \frac{6 n}{\delta}
    \\
    & \max_{w \in V} \|C_{M,\lambda}^{-1/2} N_{k,w} \| \leq  
    2 \sqrt{B} \Big( \frac{ \kappa}{\sqrt{\lambda}m} + \sqrt{ \frac{2 \sqrt{p} \mathcal{N}_{M}(\lambda) }{m}}\Big)\log \frac{6 n }{\delta}\\
    & \quad +  4 \kappa \Big( \frac{2 \kappa}{m \sqrt{\lambda}} + \sqrt{\frac{ \mathcal{N}_{M}(\lambda)}{m}}\Big) 
    \Big( 1 + \sqrt{ \frac{9 }{M} \log \frac{3 M n }{\delta}}\big( \sqrt{\eta t \kappa}\vee 1 \big) \Big)\log \frac{6 n }{\delta}
\end{align*}
Meanwhile, under the same assumptions with probability greater than $1-\delta$ for $k \geq 1$
\begin{align*}
    & \|C_{M,\lambda}^{-1/2}( C_{M} - \widehat{C}_{M})\| \leq 
    2 \kappa \Big( \frac{2 \kappa}{nm \sqrt{\lambda}} + \sqrt{\frac{ \mathcal{N}_{M}(\lambda)}{nm}}\Big) \log \frac{2}{\delta}
    \\
    & \|C_{M,\lambda}^{-1/2} N_{k} \| \leq  
    2 \sqrt{B} \Big( \frac{\kappa}{\sqrt{\lambda}nm} + \sqrt{ \frac{2 \sqrt{p} \mathcal{N}_{M}(\lambda) }{nm}}\Big)\log \frac{6}{\delta}
     \\
    & \quad + 
    4 \kappa \Big( \frac{2 \kappa}{nm \sqrt{\lambda}} + \sqrt{\frac{ \mathcal{N}_{M}(\lambda)}{nm}}\Big)  
    \Big( 1 + \sqrt{ \frac{9   }{M} \log \frac{3 M}{\delta} } \big( \sqrt{\eta t \kappa}\vee 1 \big) \Big) \log \frac{6}{\delta}
\end{align*}
\end{lemma}
The proof for this result is given in Section \ref{sec:lem:conc}. Lemma \ref{lem:conc} will be used extensively within the following analysis. To save on the burden of notation we define the following two functions for $\lambda > 0$, $K \in \mathbb{N}_{+}$ and $\delta \in (0,1]$
\begin{align*}
   g(\lambda, K) 
   & = 
   2 \kappa \Big( \frac{2 \kappa}{K \sqrt{\lambda}} + \sqrt{\frac{ \mathcal{N}_{M}(\lambda)}{K}}\Big) \\
   f(\lambda,K,\delta) & = 
   2 \sqrt{B} \Big( \frac{\kappa}{\sqrt{\lambda }K} + \sqrt{ \frac{2 \sqrt{p} \mathcal{N}_{M}(\lambda) }{K}}\Big)
     \\
    & \quad\quad + 
    4 \kappa \Big( \frac{2 \kappa}{K \sqrt{\lambda}} + \sqrt{\frac{ \mathcal{N}_{M}(\lambda)}{K}}\Big)  
    \Big( 1 + \sqrt{ \frac{9   }{M} \log \frac{3 M}{\delta} } \big( \sqrt{\eta t \kappa}\vee 1 \big) \Big).
\end{align*}
Looking to Lemma \ref{lem:conc} we note the function $g$ is associated to the high probability bound on the difference between the covariance operators, for instance  $C_{M,\lambda}^{-1/2}(C_{M} - \widehat{C}_{M})$, meanwhile $f$ is associated to the bound on the error terms, for instance $C_{M,\lambda}^{-1/2}N_{k}$.

\subsection{Bounding $\bf{E_1}$}
\label{sec:E1}
The bound for $\bf{E_1}$ is then summarised within the following Lemma. 
\begin{lemma}[Bounding $\bf{E}_1$]
\label{lem:E_1}
Let $\delta \in (0,1]$, $n,m,M \in \mathbb{N}_{+}$ and $\eta\kappa^2 \leq 1$ and $t \geq 2t^{\star} \geq 2$ and $\lambda,\lambda^{\prime} > 0$. Under assumption \ref{ass:FeatureRegularity},\ref{ass:RKHS} and \ref{ass:moment}  we have with probability greater than $1-\delta$
\begin{align*}
& \textbf{E}_{1}
 \leq \Big( 
 \|C_{M,\lambda^{\prime}}^{1/2}\| \sigma_2^{t^{\star}}t\kappa^{-1} 
 f(\lambda^{\prime},m,\delta/(2n))
+ 
20 \log(t^{\star}) (1  \vee \sqrt{\lambda \eta t^{\star}})
f(\lambda,m,\delta/(2n))
\Big)\log\frac{12n}{\delta}
\end{align*}

\end{lemma}
\begin{proof}[Lemma \ref{lem:E_1}]
Splitting the series at $1 \leq k \leq t-t^{\star}$ we have the following 
\begin{align*}
    & \textbf{E}_1 \leq 
     \Big(  
    \max_{1 \leq k \leq t, w \in V} \| N_{k,w}\| \Big) 
    \underbrace{ \sum_{k=1}^{t-t^{\star}} \eta \sum_{w \in V} |\Delta^{t-k}(v,w)| \|C_{M}^{1/2}(I - \eta C_{M})^{t-k}\|}_{\textbf{E}_{11} }  \\
    & \quad + 
    \Big( \max_{1 \leq k \leq t, w \in V} \|C_{M,\lambda}^{-1/2} N_{k,w}\| \Big)  
    \underbrace{ 
    \sum_{k=t-t^{\star} + 1}^{t} \eta \sum_{w \in V} |\Delta^{t-k}(v,w)| \|C_{M}^{1/2}(I - \eta C_{M})^{t-k}C_{M,\lambda}^{1/2} \| }_{\textbf{E}_{12}}
\end{align*}
To bound $\textbf{E}_{11}$ utilise the mixing properties of the matrix P through Lemma \ref{lem:SpectralBound}. With $\eta \kappa^2 \leq 1$ ensuring that $\eta\|C_{M}^{1/2}(I - \eta C_{M})^{t-k}\| \leq \eta \|C_{M}^{1/2}\| \leq \sqrt{\eta}  \leq \kappa^{-1} $, we arrive at the bound 
\begin{align*}
    \textbf{E}_{11} \leq \kappa^{-1} \sum_{k=1}^{t-t^{\star}} \sigma_{2}^{ t-k} \leq \sigma_2^{t^{\star}}t\kappa^{-1}.
\end{align*}
Meanwhile to bound $\textbf{E}_{12}$ utilise the contraction of the gradients, that is Lemma \ref{Lem:OperatorNorm} remark with $a=1/2$ and $\mathcal{L}= C_{M}$. With $\sum_{w \in V} | \Delta^{t-k}(v,w)| \leq 2$ this allows us to say 
\begin{align*}
    \textbf{E}_{12} & \leq 2 \eta \sum_{k=t-t^{\star}+1}^{t}\| C_{M}(I - \eta C_{M})^{t-k}\|  
    + 2 \eta \sqrt{\lambda} \sum_{k=t-t^{\star}+1}^{t} \|C_{M}^{1/2}(I - \eta C_{M})^{t-k}\|  \\
    & \leq 20 \log(t^{\star}) (1  \vee \sqrt{\lambda \eta t^{\star}}).
\end{align*}
Bounding  $\max_{1 \leq k \leq t, w \in V} \| N_{k,w}\|  \leq \|C_{M,\lambda^{\prime}}^{1/2}\| \max_{1 \leq k \leq t, w \in V} \| C_{M,\lambda^{\prime}}^{-1/2} N_{k,w}\| $ and plugging in high probability bounds for both $\max_{1 \leq k \leq t, w \in V} \| C_{M,\lambda^{\prime}}^{-1/2} N_{k,w}\|$ and $\max_{1 \leq k \leq t, w \in V} \| C_{M,\lambda}^{-1/2} N_{k,w}\|$ from Lemma \ref{lem:conc} yields the result. 
\end{proof}

\subsection{Bounding $\textbf{E}_{2}$}
\label{sec:E2}
The bound for this term utilises the following Lemma to bound operator $\|C_{M}^{1/2}\Pi^{\Delta}(w_{t:k})\|$. To save on national burden, we define the following random quantity for $\lambda > 0$
\begin{align*}
    \Delta_{\lambda} := \max_{v \in V} 
    \|C_{M}^{-1/2}( C_{M} - \widehat{C}_{M}^{(v)} )\|.
\end{align*}
We begin with the following Lemma which rewrites the norm of $\Pi^{\Delta}(w_{t:1})$ for any path $w_{t:1}$ as a series of contractions. 
\begin{lemma}
\label{lem:OperatorNormBound:Delta1}
Let $N \in \mathbb{R}^{M}$ and $w_{t:1} \in V^{t}$ and $\eta \kappa^2 \leq 1$.  Then for $u \in [0,1/2]$
\begin{align*}
    \| C_{M}^{1/2-u} \Pi^{\Delta}(w_{t:1}) N \|
    \leq 
    2 \eta \Delta_{\lambda} \|N\| \sum_{\ell=1}^{t} \|C_{M}^{1/2-u} (I - \eta C_{M})^{t-\ell}  C_{M,\lambda}^{1/2} \|
\end{align*}
\end{lemma}
Given this Lemma we present the high probability bound for $\textbf{E}_{2}$. 
\begin{lemma}[Bounding $\textbf{E}_{2}$]
\label{lem:E2}
Let $\delta \in (0,1]$, $n,m,M \in \mathbb{N}_{+}$ and $\eta\kappa^2 \leq 1$ and $t \geq 2t^{\star} \geq 2$ and $\lambda,\lambda^{\prime} > 0$. Under assumption \ref{ass:FeatureRegularity},\ref{ass:RKHS} and \ref{ass:moment}  we have with probability greater than $1-\delta$
\begin{align*}
    & \textbf{E}_{2} \leq 
    40 \kappa \|C_{M,\lambda^{\prime}}^{1/2}\| \eta t^{\star}
    \log(t)(1 \vee \sqrt{\lambda \eta t})
     \log^2 \frac{12n}{\delta}
     g(\lambda,m)
     f(\lambda^{\prime},m,\delta/(2n))
\end{align*}
\end{lemma}

\begin{proof}[Lemma \ref{lem:E2}]
Using Lemma \ref{lem:OperatorNormBound:Delta1} with $u=0$ we have for any $ t \geq k \geq t-2t^{\star}$ and $w_{t:k} \in V^{t-k+1}$
\begin{align*}
    &  \| C_{M}^{1/2}\Pi^{\Delta}(w_{t:k+1})
    N_{k,w_{k}}\| 
    \leq 
    2 \eta \Delta_{\lambda} \| N_{k,w_{k}}\| \sum_{\ell=1}^{t-k} \|C_{M}^{1/2} (I - \eta C_{M})^{t - k -\ell}  C_{M,\lambda}^{1/2} \| \\
    & \leq 
    2 \eta \Delta_{\lambda}
    \| N_{k,w_{k}}\|
    \Big( \sum_{\ell=1}^{t-k}\|C_{M} (I - \eta C_{M})^{t - k -\ell} \|
    + \sqrt{\lambda}
    \sum_{\ell=1}^{t-k}\|C_{M}^{1/2} (I - \eta C_{M})^{t - k -\ell} \|
    \Big) \\
    & \leq 
    20 \eta \Delta_{\lambda}
    \| N_{k,w_{k}}\|
    \log(t)(1 \vee \sqrt{\lambda \eta t})
\end{align*}
where we applied Lemma \ref{Lem:OperatorNorm} remark \ref{remark:Contraction} to the bound the series of contractions. The case $k=t$ the above quantity is zero. With $\sum_{w_{t:k} \in V^{t-k+1}} |\Delta(w_{t:k})| \leq 2$ this leads to the error term being bounded 
\begin{align*}
    \textbf{E}_{2} 
    \leq 
    40  \Delta_{\lambda}
    \log(t)(1 \vee \sqrt{\lambda \eta t})
     \eta t^{\star}
    \Big( \max_{1 \leq k \leq t, w \in V } \|N_{k,w}\| \Big).
\end{align*}
The final bound is arrived at by bounding for $\lambda^{\prime} > 0 $ the error term in the brackets as $ \max_{1 \leq k \leq t, w \in V } \|N_{k,w}\| \leq \|C_{M,\lambda^{\prime}}^{1/2}\|  \max_{1 \leq k \leq t, w \in V } \|C_{M,\lambda^{\prime}}^{-1/2} N_{k,w}\|$, and plugging in high probability bounds for $\max_{1 \leq k \leq t, w \in V } \|C_{M,\lambda^{\prime}}^{-1/2} N_{k,w}\|$ and $\Delta_{\lambda}$ from Lemma \ref{lem:conc}, with a union bound.  
\end{proof}

\subsection{Bounding $\textbf{E}_{3}$ }
\label{sec:E3}
The bound for this error term is similar to $\textbf{E}_{2}$ and will be presented within the following Lemma. 
\begin{lemma}[Bounding $\textbf{E}_{3}$]
\label{lem:E3}
Let $\delta \in (0,1]$, $n,m,M \in \mathbb{N}_{+}$ and $\eta\kappa^2 \leq 1$ and $t \geq 2t^{\star} \geq 2$ and $\lambda,\lambda^{\prime} > 0$. Under assumption \ref{ass:FeatureRegularity},\ref{ass:RKHS} and \ref{ass:moment}  we have with probability greater than $1-\delta$
\begin{align*}
    & \textbf{E}_{3}
    \leq 
    24 \|C_{M}^{1/2}\| 
    \|C_{M,\lambda^{\prime}}^{1/2}\|
    (\eta t) 
    \sqrt{\eta t^{\star}}
    \big(1 \vee \sqrt{\lambda \eta t^{\star}} \big)
     \log^2 \frac{12n}{\delta}
     g(\lambda,m) f(\lambda^{\prime},m,\delta/(2n))
\end{align*}
\end{lemma}
\begin{proof}[Lemma \ref{lem:E3}]
For $1 \leq k \leq t-2t^{\star} -1$ and $w_{t:k} \in V^{t-k+1}$ use Lemma \ref{lem:OperatorNormBound:Delta1} with $u=1/2$ as well as $\eta \kappa^2 \leq 1$ to bound with $\lambda > 0 $
\begin{align*}
    & \|C_{M}^{1/2} \Pi(w_{t:k+t^{\star} +1 }) \Pi^{\Delta}(w_{k+t^{\star}: k +1}) 
    N_{k,w_{k}}\|\\
    & \leq 
    \|C_{M}^{1/2}\|\|\Pi^{\Delta}(w_{k+t^{\star}: k +1}) 
    N_{k,w_{k}}\| \\
    & \leq 
    2 \eta \|C_{M}^{1/2}\| 
    \Delta_{\lambda} \|N_{k,w_{k}}\| 
    \sum_{\ell=1}^{t^{\star}} \|(I - \eta C_{M})^{t^{\star} - \ell}C_{M,\lambda}^{1/2} \| \\
    & \leq 
    2  \|C_{M}^{1/2}\| 
    \Delta_{\lambda} \|N_{k,w_{k}}\| 
    \Big(  \eta \sum_{\ell=1}^{t^{\star}}\|(I - \eta C_{M})^{t^{\star} - \ell}C_{M}^{1/2} \| 
    + 
    \sqrt{\lambda} \eta t^{\star}
    \Big) \\
    & \leq 
    12  \|C_{M}^{1/2}\| 
    \Delta_{\lambda} \|N_{k,w_{k}}\| 
    \sqrt{\eta t^{\star}}
    (1 \vee \sqrt{\lambda \eta t^{\star}} )
\end{align*}
where we have bounded the series of contractions using Lemma \ref{Lem:OperatorNorm} remark \ref{remark:Contraction} once again. With $\sum_{w_{t:k} \in V^{t-k+1}} |\Delta(w_{t:k})| \leq 2$, plugging in the above yields the bound for $\textbf{E}_{3}$
\begin{align*}
    \textbf{E}_{3}
    \leq 24 \|C_{M}^{1/2}\| 
    (\eta t) 
    \sqrt{\eta t^{\star}}
    \Delta_{\lambda} 
    (1 \vee \sqrt{\lambda \eta t^{\star}} )
    \Big( 
    \max_{1 \leq k \leq t, w \in V} \|N_{k,w}\| 
    \Big).
\end{align*}
The final bound is arrived at by bounding $\Delta_{\lambda} $ and $\Big(  \max_{1 \leq k \leq t, w \in V} \|N_{k,w}\| \Big)$ in an identical manner to Lemma \ref{lem:E2} for error term $\textbf{E}_{2}$. 
\end{proof}
\subsection{Bounding $\textbf{E}_{4}$}
\label{sec:E4}
This term will be controlled through the convergence of $P^{t^{\star}}$ to the stationary distribution. It is summarised within the following Lemma. 
\begin{lemma}[Bounding $\textbf{E}_{4}$]
\label{lem:E4}
Let $\delta \in (0,1]$, $n,m,M \in \mathbb{N}_{+}$ and $\eta\kappa^2 \leq 1$ and $t \geq 2t^{\star} \geq 2$ and $\lambda  > 0$. Under assumption \ref{ass:FeatureRegularity},\ref{ass:RKHS} and \ref{ass:moment}  we have with probability greater than $1-\delta$
\begin{align*}
    & \textbf{E}_{4} 
    \leq 
    4  \|C_{M,\lambda}^{1/2}\| 
    \big( \sqrt{n} \sigma_2^{t^{\star}} \wedge 1 \big)
    (\eta t)\log \frac{6n}{\delta}
    f(\lambda,m,\delta/n)
\end{align*}

\end{lemma}
\begin{proof}[Lemma \ref{lem:E4}]
Begin by bounding for $t-2t^{\star}-1 \geq k \geq 1$ , $w_{k} \in V$ and $w_{t:k+t^{\star} +1} \in V^{t-t^{\star} -k}$ the following 
\begin{align*}
    \| 
    C_{M}^{1/2} 
    \Pi^{\Delta}(w_{t:k+t^{\star}+1}) (I - \eta C_{M} )^{t^{\star}}
    N_{k,w_{k}}\|  
    \leq 
    2 \|C_{M}^{1/2} \|
    \|N_{k,w_{k}}\|.
\end{align*}
Furthermore, we can bound the summation over paths by the deviation of the form $\sum_{w \in V}  |\Delta^{t^{\star}}(v,w )| $ and use Lemma \ref{lem:SpectralBound} thereafter to arrive at 
\begin{align*}
   &  \sum_{w_{k} \in V} 
    \sum_{w_{t:k+t^{\star} +1} \in V^{t-t^{\star} -k}}
    | P_{vw_{t:k+t^{\star}+1}}  \Delta^{t^{\star}}(w_{k+t^{\star} +1},w_{k})| \\
    & = \sum_{w_{t:k+t^{\star} +1} \in V^{t-t^{\star} -k}}
    | P_{vw_{t:k+t^{\star}+1}}| \Big( \sum_{w_{k} \in V}  |\Delta^{t^{\star}}(w_{k+t^{\star} +1},w_{k})| \Big) \\
    & \leq 
    \max_{u \in V} 
    \Big( \sum_{w \in V}  |\Delta^{t^{\star}}(u,w )| \Big)
    \Big( \sum_{w_{t:k+t^{\star} +1} \in V^{t-t^{\star} -k}}| P_{vw_{t:k+t^{\star}+1}}| \Big) \\
    & =
    \max_{u \in V} 
    \Big( \sum_{w \in V}  |\Delta^{t^{\star}}(u ,w )| \Big) \\
    & \leq 
    2 \big( \sqrt{n} \sigma_2^{t^{\star}} \wedge 1 \big).
\end{align*}
Bringing everything together yields the following bound for $\textbf{E}_{4}$
\begin{align}
    \textbf{E}_{4}
    \leq 
    2 \big( \sqrt{n} \sigma_2^{t^{\star}} \wedge 1 \big)
    (\eta t) 
    \Big( \max_{1 \leq k \leq t, w \in V} \| N_{k,w}\|\Big)
\end{align}
Plugging in high probability bounds for $\max_{1 \leq k \leq t, w \in V} \| N_{k,w}\|$ following Lemma \ref{lem:E2} for error term $\textbf{E}_{2}$ then yields the bound. 
\end{proof}

\subsection{Bounding $\textbf{E}_{5}$}
\label{sec:E5}
The summation over paths in this case is decoupled from the error. This allows for a more sophisticated bound to be applied, which considers the deviation of the iterates from the average. 
The following Lemma effectively bounds the norm of $\sum_{w_{t:1} \in V^{t}} \Delta(w_{t:1}) \Pi^{\Delta}(w_{t:1})$, which involves a sum over the paths $w_{t:1}$. 
\begin{lemma}
\label{lem:OperatorNormBound:Delta2}
Let $N \in \mathbb{R}^{M}$, $w_{t:1} \in V^t$ and $\lambda_i \geq 0$ for $i \in \{ 1,2,3 \}$. Then, 
\begin{align*}
    & \| \sum_{w_{t:1} \in V^{t}} \Delta(w_{t:1}) C_{M}^{1/2} \Pi^{\Delta}(w_{t:1}) N\|
    \leq 
    4 \eta \Delta_{\lambda_1}\|N\|
    \sum_{k=1}^{t} 
    \|C_{M}^{1/2} (I - \eta \widehat{C}_{M})^{t-k}
    C_{M,\lambda_1}^{1/2}\|
    ( \sigma_2^{t-k+1}\wedge 1 ) \\
    & \quad + 
    8  \eta^2 \Delta_{\lambda_{2}}\Delta_{\lambda_3} \|N\|
    \sum_{k=2}^{t} \sum_{\ell=1}^{k-1} 
    \|C_{M}^{1/2} (I - \eta \widehat{C}_{M} )^{t-k}C_{M,\lambda_2}^{1/2}\|
    \| (I - \eta \widehat{C}_{M} )^{k-1-\ell} C_{M,\lambda_{3}}^{1/2} \|
    (\sigma_2^{k - \ell} \wedge 1) 
\end{align*}
\end{lemma}

The bound for this error term is then summarised within the following Lemma. 

\begin{lemma}[Bounding $\textbf{E}_{5}$]
\label{lem:E5}
Let $\delta \in (0,1]$, $n,m,M \in \mathbb{N}_{+}$ and $\eta\kappa^2 \leq 1$ and $t \geq 2t^{\star} \geq 2$ and $\lambda^{\prime},\lambda_{i}  > 0$ for $i =1,\dots,3$. Under assumption \ref{ass:FeatureRegularity},\ref{ass:RKHS} and \ref{ass:moment} and if $\frac{9 \kappa^2}{M} \log \frac{M}{\delta} \leq \lambda_i$  for $i=1,2$ then with probability greater than $1-8 \delta$
\begin{align*}
    \textbf{E}_{5} \leq \textbf{E}_{51}  + \textbf{E}_{52} 
\end{align*}
where 
\begin{align*}
    \textbf{E}_{51} 
    & \leq 
    84 \|C_{M}^{1/2} C_{M,\lambda_1}^{1/2}\| \|C_{M,\lambda^{\prime}}^{1/2}\| 
    \eta t ( 1 \vee \sigma_2^{t^{\star}} \eta t  \vee \lambda_1 \eta t^{\star}) 
    \times 
    g(\lambda_1,m)  f(\lambda^{\prime},nm,\delta)
    \log(t)
    \log^2 \frac{6 n}{\delta} 
    \\
    \textbf{E}_{52} 
    & \leq 
    160 \|C_{M,\lambda^{\prime}}^{1/2}\|  \|C_{M,\lambda_3}^{1/2}\|
    (\eta t) 
    (1 \vee \lambda_2 \eta t) (\sigma_2^{t^{\star}} \eta t \vee \eta t^{\star}) 
    \times 
     g(\lambda_2,m) g(\lambda_3,m) f(\lambda^{\prime},nm,\delta)
    \log(t) 
    \log^3 \frac{6n}{\delta}  
\end{align*}
\end{lemma}

\begin{proof}[Lemma \ref{lem:E5}]
Applying for $1 \leq k \leq t-2t^{\star}-1$ Lemma \ref{lem:OperatorNormBound:Delta2} with $N=(I - \eta C_{M})^{t^{\star}}N_{k} = N^{\prime}_{k}$, and $w_{t:k+t^{\star}+1} \in V^{t-t^{\star}-k}$ to elements within the series of $\textbf{E}_{5}$ we arrive at 
\begin{align*}
    \textbf{E}_{5}
    & \leq 
    4 \sum_{k=1}^{t-2t^{\star}-1}
    \eta^2  \Delta_{\lambda_1}\|N^{\prime}_{k} \|
    \sum_{\ell=1}^{t-t^{\star}-k } 
    \| C_{M}^{1/2} (I - \eta \widehat{C}_{M} )^{t-t^{\star} - k -\ell}
    C_{M,\lambda_1}^{1/2}\| \big( \sigma_2^{t-t^{\star}-k- \ell + 1} \wedge 1 \big)\\
    & \quad + 
    8  \sum_{k=1}^{t-2t^{\star}-1} 
    \eta^3 \Delta_{\lambda_{2}}\Delta_{\lambda_3} \|N^{\prime}_{k}\|
    \sum_{\ell=2}^{t-t^{\star}-k }  \sum_{j=1}^{\ell -1 } 
    \|C_{M}^{1/2} (I - \eta \widehat{C}_{M})^{t-t^{\star}-k -\ell}C_{M,\lambda_2}^{1/2}\|
     \\
    & \quad\quad\quad\quad\quad\quad\quad\quad\quad\quad\quad\quad\quad\quad\quad\quad\quad
    \times   \| (I - \eta \widehat{C}_{M} )^{\ell - j -1} C_{M,\lambda_3}^{1/2} \| (\sigma_2^{\ell-j} \wedge 1)\\
    & = \textbf{E}_{51} + \textbf{E}_{52}
\end{align*}
where we have labelled the remaining error terms $\textbf{E}_{51},\textbf{E}_{52}$. Each of these terms are now bounded. 

To bound the first term $\textbf{E}_{51}$, begin by for $1 \leq k \leq t-2t^{\star}-1$ splitting the series at $1 \leq \ell \leq t-2t^{\star} - k$ to arrive at
\begin{align*}
    & \eta \sum_{\ell=1}^{t-t^{\star}-k } 
    \| C_{M}^{1/2} (I - \eta \widehat{C}_{M})^{t-t^{\star} - k-\ell}
    C_{M,\lambda_1}^{1/2}\| \big( \sigma_2^{t-t^{\star}-k- \ell + 1} \wedge 1 \big)  \\
    & \leq 
    \|C_{M}^{1/2}C_{M,\lambda_1}^{1/2}\|   \eta 
    \sum_{\ell=1}^{t-2 t^{\star}-k }   \big( \sigma_2^{t-t^{\star}-k- \ell + 1} \wedge 1 \big)
    + 
    \eta \sum_{\ell=t-2 t^{\star}-k}^{t-t^{\star} - k }
    \| C_{M}^{1/2} (I - \eta \widehat{C}_{M})^{t-t^{\star} - k - \ell  }
    C_{M,\lambda_1}^{1/2}\| \\
    & \leq 
    \|C_{M}^{1/2}C_{M,\lambda_1}^{1/2}\|   \eta 
    \sum_{\ell=1}^{t-2 t^{\star}-k }   \big( \sigma_2^{t-t^{\star}-k- \ell + 1} \wedge 1 \big)\\
    & \quad\quad  
    + 
    \eta \|C_{M}^{1/2} \widehat{C}_{M,\lambda_1}^{-1/2} \| \|\widehat{C}_{M,\lambda_1}^{-1/2} C_{M,\lambda_1}^{1/2} \| 
    \sum_{\ell=t-2 t^{\star}-k}^{t-t^{\star} - k}
    \| \widehat{C}_{M,\lambda_1}^{1/2} (I - \eta \widehat{C}_{M})^{t-t^{\star} - k - \ell } \widehat{C}_{M,\lambda_1}^{1/2}\| \\
    & \leq 
    \|C_{M}^{1/2}C_{M,\lambda_1}^{1/2}\| \sigma_2^{t^{\star}} \eta t 
    + 10 \|C_{M}^{1/2} \widehat{C}_{M,\lambda_1}^{-1/2} \| \|\widehat{C}_{M,\lambda_1}^{-1/2} C_{M,\lambda_1}^{1/2} \| 
    \log(t)(1 \vee \lambda_1 \eta t^{\star})
\end{align*}
where for the first series used  that $\sigma_2^{t-t^{\star}-k- \ell + 1} \leq \sigma_2^{t^{\star}}$ from $\ell \leq t-2t^{\star}-k$ meanwhile for the second series 
\begin{align*}
    & \eta \sum_{\ell=t-2 t^{\star}-k}^{t-t^{\star} - k}
    \| \widehat{C}_{M,\lambda_1}^{1/2} (I - \eta \widehat{C}_{M})^{t-t^{\star} - k - \ell  } \widehat{C}_{M,\lambda_1}^{1/2}\|\\
    & \leq 
    \eta \sum_{\ell=t-2 t^{\star}-k}^{t-t^{\star} - k}\| \widehat{C}_{M} (I - \eta \widehat{C}_{M})^{t-t^{\star} - k - \ell  }\|
    + 
    \eta \lambda_1  
    \sum_{\ell=t-2 t^{\star}-k}^{t-t^{\star} - k}\|(I - \eta \widehat{C}_{M})^{t-t^{\star} - k - \ell }\| \\
    & \leq 5 \log (t) + 5 \lambda_1 \eta t^{\star} 
\end{align*}
to which we applied Lemma \ref{Lem:OperatorNorm} remark \ref{remark:Contraction} to bound the series of contractions. 
This leads to the bound for $\textbf{E}_{51}$  
\begin{align*}
    \textbf{E}_{51}
    \leq 
    4 \Delta_{\lambda_1}
    \eta t 
    \Big( \|C_{M}^{1/2}C_{M,\lambda_1}^{1/2}\| \sigma_2^{t^{\star}} \eta t 
    + 10 \|C_{M}^{1/2} \widehat{C}_{M,\lambda_1}^{-1/2} \| \|\widehat{C}_{M,\lambda_1}^{-1/2} C_{M,\lambda_1}^{1/2} \|  
    \log(t)(1 \vee \lambda_1 \eta t^{\star}) \Big) 
    \big( \max_{1 \leq k \leq t} \|N^{\prime}_{k}\| \big). 
\end{align*}
Provided $\frac{9 \kappa^2}{M} \log \frac{M}{\delta} \leq \lambda_{1}$ we have from Lemma 3 in \cite{carratino2018learning} that with probability greater than $1- \delta$ 
\begin{align*}
    \|C_{M}^{1/2}\widehat{C}_{M,\lambda_1}^{-1/2}\|
    \|\widehat{C}_{M,\lambda_1}^{-1/2}C_{M,\lambda_1}^{1/2}\|
    \leq 
    \|\widehat{C}_{M,\lambda_1}^{-1/2}C_{M,\lambda_1}^{1/2}\|^2
    \leq 2.
\end{align*}
Meanwhile for $\lambda^{\prime} > 0$, we can bound \\ 
$ \max_{1 \leq k \leq t} \|N^{\prime}_{k}\|  \leq  \|C_{M,\lambda^{\prime}}^{1/2} \| 
 \max_{1 \leq k \leq t} \|C_{M,\lambda^{\prime}}^{-1/2} N^{\prime}_{k}\| \leq \|C_{M,\lambda^{\prime}}^{1/2} \|\max_{1 \leq k \leq t} \|C_{M,\lambda^{\prime}}^{-1/2} N_{k}\|$. The bound is arrived at by also plugging in high probability bounds for $\|C_{M,\lambda^{\prime}}^{-1/2} N_{k}\|$ and $\Delta_{\lambda_1}$ from Lemma \ref{lem:conc}.
 
Finally to bound $\textbf{E}_{52}$. Begin by bounding for $1 \leq k \leq t-2t^{\star} -1 $ as well as $2 \leq \ell \leq t^{\star}$ the series as 
\begin{align*}
    \sum_{j=1}^{\ell - 1 }  \| (I - \eta \widehat{C}_{M} )^{\ell - j} C_{M,\lambda_3}^{1/2} \| (\sigma_2^{\ell-j} \wedge 1)
    \leq 
    \| C_{M,\lambda_3}^{1/2}\|t^{\star}. 
\end{align*}
Meanwhile for $t^{\star} + 1 \leq \ell \leq t-t^{\star} - k $ we can split the series as $1 \leq j \leq \ell - t^{\star}$ 
\begin{align*}
    & \sum_{j=1}^{\ell -1 }  \| (I - \eta \widehat{C}_{M} )^{\ell - j} C_{M,\lambda_3}^{1/2} \| (\sigma_2^{\ell-j} \wedge 1)\\
     & \leq 
    \| C_{M,\lambda_3}^{1/2} \| \sum_{j=1}^{\ell - t^{\star} }  
     (\sigma_2^{\ell-j} \wedge 1) 
     + 
    \sum_{j=\ell - t^{\star} + 1}^{ \ell - 1 }
    \| (I - \eta \widehat{C}_{M} )^{\ell - j} C_{M,\lambda_3}^{1/2} \| \\
    & \leq 
    \| C_{M,\lambda_3}^{1/2} \| ( \sigma_2^{t^{\star}} t  + t^{\star})
\end{align*}
where for the first series we applied $j \leq \ell - t^{\star}$ to say $\sigma_2^{\ell-j} \leq \sigma_2^{t^{\star}}$, and for the second simply summed up the $t^{\star}$ terms after bounding   $\| (I - \eta \widehat{C}_{M} )^{\ell - j} C_{M,\lambda_3}^{1/2} \| \leq  \|C_{M,\lambda_3}^{1/2} \|$. Plugging in the above bound for all $2 \leq \ell \leq t-t^{\star} - k$ we arrive at the following bound for $\textbf{E}_{52}$
\begin{align*}
    \textbf{E}_{52}
    & \leq 
     8 \Delta_{\lambda_{2}}\Delta_{\lambda_{3}} 
    \big( \max_{1 \leq k \leq t} \|N_{k}^{\prime}\| \big)
    \| C_{M,\lambda_3}^{1/2} \| ( \sigma_2^{t^{\star}} \eta t  + \eta t^{\star}) 
    \sum_{k=1}^{t-2t^{\star} - 1} 
    \eta^2
    \sum_{\ell=2}^{t-t^{\star} - k}
    \|C_{M}^{1/2} (I - \eta \widehat{C}_{M})^{t-t^{\star}-k -\ell}C_{M,\lambda_2}^{1/2}\| 
\end{align*}
For $1 \leq k \leq t-2t^{\star}-1$ the series of contractions over $\ell$ can be bounded using Lemma \ref{Lem:OperatorNorm} remark \ref{remark:Contraction} in a similar manner to previously as 
\begin{align*}
    \eta \sum_{\ell=2}^{t-t^{\star} - k}
    \|C_{M}^{1/2} (I - \eta \widehat{C}_{M})^{t-t^{\star}-k -\ell} C_{M,\lambda_2}^{1/2}\| 
    \leq 
    \| C_{M}^{1/2} \widehat{C}_{M,\lambda_2}^{-1/2} \| 
    \|\widehat{C}_{M,\lambda_2}^{-1/2}C_{M,\lambda_2}^{1/2}\|
    10 \log(t) (1 \vee \lambda_{2} \eta t). 
\end{align*}
Summing up the remaining series for over $k$, using that $ \| C_{M}^{1/2} \widehat{C}_{M,\lambda_2}^{-1/2} \| 
\|\widehat{C}_{M,\lambda_2}^{-1/2}C_{M,\lambda_2}^{1/2}\| \leq 2$ from $\frac{9 \kappa^2}{M} \log \frac{M}{\delta} \leq \lambda_2$, plugging in high probability bounds for $\max_{1 \leq k \leq t} \|N_{k}^{\prime}\|$ from the the error term $\textbf{E}_{51}$, as well as high probability bounds for $\Delta_{\lambda_2},\Delta_{\lambda_3}$ from Lemma \ref{lem:conc} yields the bound. 
\end{proof}

\section{Final bounds}
In this section we bring together the high probability bounds for the Statistical Error and Distributed Error. This section is then as follows. Section \ref{sec:proof:Refined} provides the proof for Theorem \ref{thm:WorstCase}. Section  \ref{sec:proof:WorstCast} gives the proof for Theorem \ref{thm:WorstCase}.

\subsection{Refined Bound (Theorem \ref{thm:Refined})}
\label{sec:proof:Refined}
In this section we give conditions under which we obtained a refined bound. 
\begin{proof}[Theorem \ref{thm:Refined}]
Fixing $\delta \in (0,1]$ and a constant $c_{\text{union}} > 1 $, assume that 
\begin{align*}
    \eta t & = (nm)^{\frac{1}{2r+\gamma}} \\
    M & \geq \Big( (nm)^{\frac{1 + \gamma(2r-1) }{2r+\gamma}} \Big) \vee \Big( \eta t \log \frac{60 n \kappa^2 ( \eta t \vee M) c_{\text{union}}  }{\delta} \Big)  \\
    t^{\star} & \geq 2 \frac{\log(nm t)}{1-\sigma_2} \\
    m & \geq \Big( ( 1 \vee (\eta t^{\star}))^{2r+\gamma} n^{2r/\gamma}\Big) \vee \Big( ( 1 \vee (\eta t^{\star}))^{2} n \Big) 
    \vee 
    \Big( (1 \vee \eta t^{\star})^{\frac{ (1+\gamma)(2r+\gamma)}{2(r+\gamma-1)}} n^{\frac{(r+1)}{(r+\gamma - 1)}}\Big)
\end{align*}
Now, consider the error decomposition given \eqref{equ:ErrorDecomp}, to arrive at the bound
\begin{align*}
    \mathcal{E}(f_{t+1,v}) - \mathcal{E}(f_{\mathcal{H}}) 
    \leq 
    2 \underbrace{ 
    \| S_{M} \widehat{\omega}_{t+1,v} - S_{M}\widehat{v}_{t} \|_{\rho}^2 
    }_{( \text{Network Error})^2}
    + 
    2\underbrace{ 
    \| S_{M}\widehat{v}_{t}  - P f_{\rho} \|_{\rho}^2 
    }_{ (\text{Statistical Error})^2}.
\end{align*}
Begin by bounding the statistical error by using Lemma \ref{lem:StatisticalError}. Using Assumption \ref{ass:SourceCap} to bound $\mathcal{N}(\frac{1}{\eta t}) \leq Q^2 (\eta t)^{\gamma}$, and noting that $M \geq (4 + 18 \eta t \kappa^2) \log \frac{60 \kappa^2 \eta t}{\delta}$ is satisfied, allows us to upper bound with probability greater than $1-\delta$
\begin{align*}
    & \| S_{M}\widehat{v}_{t}  - P f_{\rho} \|_{\rho}^2 
    \leq 
    (nm)^{-2r/(2r+\gamma)}
    \Big( 
    c_1^2 \Big( 1 \vee \frac{ (\eta t ) \log \frac{3 M}{\delta}}{M} \Big) (1 \vee Q^2)  
    \log^2(t) \log^2 \Big( \frac{12}{\delta}\Big)  
    + c_3^2 
    \Big) 
    \\
    & \quad 
     + c_2^2 
    \Big( \frac{1}{M^{2r}}  \vee \frac{ Q^2 }{ M (nm )^{(1-\gamma) (2r-1 )/(2r+\gamma) } } \Big)
    \log^{2(1-r)}( 11 \kappa^2 \eta t) 
    \log^2 \big( \frac{6}{ \delta}\big) 
\end{align*}
The quantity within the brackets for second term is then upper bounded $ \frac{1}{ M (nm )^{(1-\gamma) (2r-1 )/(2r+\gamma) }} \leq (nm)^{-2r/(2r+\gamma)}$ provided $M \geq (nm)^{\frac{1 + \gamma(2r-1)}{2r+\gamma}}$, which is satisfied as an assumption in the Theorem. This results in an upper bound on the statistical error that is, up to log factors, decreasing as $(nm)^{-2r/(2r+\gamma)}$ in high probability. 

We now proceed to bound the \textit{Network Error} Term. Begin by considering error decomposition given in \eqref{equ:NetworkErrorDecomp} into the terms $\textbf{E}_{1},\textbf{E}_2,\textbf{E}_3,\textbf{E}_{4},\textbf{E}_{5}$, in particular by applying the inequality $(a + b)^{2} \leq 2 a^2 + 2b^2$  multiple times we get 
\begin{align*}
    \| S_{M} \widehat{\omega}_{t+1,v} - S_{M}\widehat{v}_{t} \|_{\rho}^2 
    \leq 2 \textbf{E}_1^2 + 4 \textbf{E}_2^2 + 8 \textbf{E}_3^2 + 16 \textbf{E}_4^2 + 32 \textbf{E}_5^2,
\end{align*}
and thus it is sufficient to show each of these terms is decreasing as $(nm)^{-2r/(2r+\gamma)}$ in high probability. Before doing so we note Lemma 4  in \cite{carratino2018learning} states for any $\lambda >0$ that if
\begin{align*}
    M \geq \big( 4 + \frac{18 \kappa^2}{\lambda} \big) \log \frac{12 \kappa^2}{\lambda \delta } 
\end{align*}
then with probability greater than $1-\delta$ we have $\mathcal{N}_{M}(\lambda) \leq q \mathcal{N}(\lambda)$ where $q = \max\big( 2.55,\frac{2\kappa^2}{\|L\|} \big)$. We note this is satisfied with both $\lambda = (\eta t)^{-1},(1 \vee (\eta t^{\star}))^{-1}$ by the assumptions within the Theorem, and as such, we can interchange from $\mathcal{N}_{M}(\lambda)$ to $\mathcal{N}(\lambda)$ with at most a constant cost of $q$. 

We begin by bounding $\textbf{E}_1^2$ by considering Lemma \ref{lem:E_1} with $\lambda^{\prime} = \kappa^2$ and $\lambda = (1 \vee \eta t^{\star})^{-1}$, which leads to with probability greater than $1-\delta$
\begin{align*}
    \textbf{E}_1^2 
 \leq \Big( 
 2 \|C_{M,\lambda^{\prime}}^{1/2}\|^2 \sigma_2^{2 t^{\star}}t^2\kappa^{-2} 
 (f(\lambda^{\prime},m,\delta/(2n)))^2
+ 
40 \log^2(t^{\star}) (f(\lambda,m,\delta/(2n)))^2
\Big)\log^2\frac{12n}{\delta}
\end{align*}
Now due to $t^{\star} \geq \frac{ 2 \log(nm t) }{1-\sigma_2} \geq  \frac{ 2 \log(nm t) }{-\log(\sigma_2)} $ (the second inequality arising from $\log(x) \geq 1- x^{-1}$ for $x \geq 0$) we have $\sigma_2^{t^{\star}} \leq (t nm)^{-2}$. As such with the fact that $f(\kappa^2, m, \delta/(2n)) \lesssim  m^{-1/2}$ in high probability, the first term above is decreasing, upto logarithmic factors, as $(nm)^{-2r/(2r+\gamma)}$. Meanwhile for the second term we have that 
\begin{align*}
    f( (1 \vee \eta t^{\star})^{-1},m,\delta/2)^2 \leq 
    a_1^2 \big( 
    \frac{(1 \vee \eta t^{\star}) }{m^2 } \vee \frac{ (1 \vee \eta t^{\star})^{\gamma}}{m}
    \big)
    \big( 1 \vee \frac{3(\eta t \kappa \vee 1) }{M} \log \frac{6 M n }{\delta} \big)
\end{align*}
for the constant $a_1 = 64 \Big( \sqrt{B}(\kappa \vee \sqrt{ \sqrt{p} q}) \Big) \vee  \big( \kappa \vee \sqrt{q} \big)$. For $\textbf{E}_1^2$ to be decreasing at the rate $(nm)^{-2r/(2r+\gamma)}$, up to logarithmic factors, we then require $\frac{ (1 \vee \eta t^{\star})^{\gamma}}{m} \leq (nm)^{-2r/(2r+\gamma)}$ which is satisfied when $m \geq (1 \vee (\eta t^{\star}))^{2r+\gamma} n^{2r/\gamma}$.

Proceed to bound $\textbf{E}_2^2$ by considering Lemma \ref{lem:E2} with $\lambda = 1/(\eta t)$ and $\lambda^{\prime} = \kappa^2$ to arrive at with probability greater than $1-\delta $
\begin{align*}
    \textbf{E}_2^2 
    \leq 
    40^2 \kappa^2 \|C_{M,\lambda^{\prime}}^{1/2}\|^2 \log^2(t) 
    (\eta t^{\star})^2 (g(\lambda,m))^2 (f(\lambda^{\prime},m,\delta/(2n)))^2
    \log^4 \frac{12 n}{\delta}
\end{align*}
As discussed previously, we have with high probability that $(f(\kappa^2,m,\delta/(2n)))^{2} \lesssim 1/m$, meanwhile 
\begin{align*}
    g((\eta t)^{-1},m)^2 \leq 
    a_2^2 \big( \frac{\eta t}{m^2} \vee \frac{(\eta t)^{\gamma}}{m} \big)
\end{align*}
where $a_2 = 8 \kappa(\kappa \vee \sqrt{q}) $. As such for $\textbf{E}_2^2$ to be decreasing at the rate $(nm)^{-2r/(2r+\gamma)}$ we require $\frac{(\eta t)^{\gamma} (1 \vee \eta t^{\star})^2 }{m^2} \leq (nm)^{-2r/(2r+\gamma)}$ which, plugging in $\eta t = (nm)^{1/(2r+\gamma)}$ is satisfied when $m \geq (1 \vee \eta t^{\star})^2 n $.

Bounding $\textbf{E}_3$ using Lemma \ref{lem:E3} with $\lambda = (1 \vee (\eta t^{\star}))^{-1}$  and $\lambda^{\prime} = \kappa^2$ we have with probability greater than $1- \delta$
\begin{align*}
    \textbf{E}_3^2 
    \leq 
    24^2 \|C_{M}^{1/2}\|^2 \|C_{M,\lambda^{\prime}}^{1/2}\|^2
    (\eta t)^2 (\eta t^{\star}) (g(\lambda,m))^2 (f(\lambda^{\prime},m,\delta/(2n)))^2
    \log^4 \frac{12n}{\delta}.
\end{align*}
Following the steps for $\textbf{E}_2$, we have with high probability that $f(\kappa^2,m,\delta/(2n))^2 \lesssim 1/m$, meanwhile $g( (1 \vee (\eta t^{\star}))^{-1},m) \lesssim (1 \vee \eta t^{\star})^{\gamma}/m$. As such for $\textbf{E}_{3}^2$ to be decreasing with the rate $ (nm)^{-2r/(2r+\gamma)}$ we require $\frac{ (\eta t)^2 (1 \vee \eta t^{\star})^{1+\gamma}}{m^2} \leq (nm)^{-2r/(2r+\gamma)}$, which is satisfied when $r+\gamma > 1$ and $m \geq (1 \vee \eta t^{\star})^{\frac{ (1+\gamma)(2r+\gamma)}{2(r+\gamma-1)}} n^{\frac{(r+1)}{(r+\gamma - 1)}}$.

Now to bound $\textbf{E}_4$ we consider Lemma \ref{lem:E4} with $\lambda = \kappa^2$ to arrive at with probability greater than $1-\delta$
\begin{align*}
    \textbf{E}_4^2 
    \leq 
    16 \|C_{M,\lambda}^{1/2}\|^2 
    (n \sigma_2^{2t^{\star}} \wedge 1) (\eta t)^2 \log^2 \Big( \frac{6n}{\delta}\Big) (f(\lambda,m,\delta/n))^2.
\end{align*}
Following the previous analysis we know with high probability $(f(\lambda,m,\delta/n))^2 = \widetilde{O}(1/m)$ and that $t^{\star}$ is such that $\sigma_2^{t^{\star}} \leq (tnm)^{-2}$. Combining these two facts we have that $\textbf{E}_{4}^{2} $ is of the order $(nm)^{-2r/(2r+\gamma)}$ with high probability. 

The bound for $\textbf{E}_5^2$ is naturally split across the terms $\textbf{E}_{51},\textbf{E}_{52}$ from Lemma \ref{lem:E5}. In particular we have that 
\begin{align*}
    \textbf{E}_{5}^2 
    \leq 
    2 \textbf{E}_{51}^2 + 2\textbf{E}_{52}^2 
\end{align*}
The remainder of the proof then shows each of the terms above are decreasing at the rate $ (nm)^{-2r/(2r+\gamma)}$ in high probability by using the bounds provided within Lemma \ref{lem:E5}. We note the condition $\frac{9\kappa^2}{M} \log \frac{M 9 \kappa^2}{\delta} \leq \lambda_i$ for $i=1,2$ is satisfied for $\lambda_1 = (1 \vee (\eta t^{\star}))^{-1}$ and $\lambda_2 = (\eta t )^{-1}$ by the assumptions.

Consider the bound for $\textbf{E}_{51}$ with $\lambda_1 = (1 \vee (\eta t^{\star}))^{-1}$ and $\lambda^{\prime} = \kappa^2$, so we have with probability greater than $1-\delta$
\begin{align*}
    \textbf{E}_{51}^2 
    \leq 
    84^2 \|C_{M}^{1/2}C_{M,\lambda_1}^{1/2}\|^2 \|C_{M,\lambda^{\prime}}^{1/2}\|^2
    (\eta t)^{2} (1 \vee \sigma_2^{2t^{\star}}(\eta t )^2) 
    (g(\lambda_1,m))^{2} (f(\lambda^{\prime},nm,\delta/8 ))^2 \log^2(t) \log^4 \frac{48 n}{\delta}. 
\end{align*}
From previously we have that $t^{\star}$ so that $\sigma_2^{t^{\star}} \leq (t nm)^{-2}$ and thus $\sigma_2^{t^{\star}}\eta t \leq 1$. Meanwhile following steps from previously we have  $(g(\lambda_1,m))^2 \lesssim (1 \vee (\eta t^{\star}))^{\gamma}/m$ as well as with high probability $(f(\lambda^{\prime},nm,\delta))^{2} \lesssim (nm)^{-1}$. As such we require $ \frac{ (\eta t)^2 (1 \vee (\eta t^{\star}))^{\gamma}}{m(nm)} \leq (nm)^{-2r/(2r+\gamma)}$ which is satisfied when  $r+\gamma > 1$ and $m \geq n^{\frac{2 - \gamma}{2(r+\gamma - 1)}} ( 1 \vee (\eta t^{\star}))^{\frac{\gamma(2r+\gamma)}{2(r+\gamma - 1)}}$. This is then implied by the assumption that  $m \geq (1 \vee \eta t^{\star})^{\frac{ (1+\gamma)(2r+\gamma)}{2(r+\gamma-1)}} n^{\frac{(r+1)}{(r+\gamma - 1)}}$ and $r + \gamma \geq 1$. 

Finally to bound $\textbf{E}_{52}$ consider the bound given with $\lambda_2 = (\eta t)^{-1}$, and $\lambda_3 = \lambda^{\prime} = \kappa^2$ to arrive at with probability greater than $1-\delta$
\begin{align*}
    \textbf{E}_{52}^2 
    \leq 
    160^2 \|C_{M,\lambda^{\prime}}^{1/2}\|^2 \|C_{M,\lambda_3}^{1/2}\|^2 
    (\eta t)^2 \big( \sigma_2^{t^{\star}} \eta t \vee (\eta t^{\star} )^2 \big)  g(\lambda_2,m) g(\lambda_3,m) f(\lambda^{\prime},nm,\delta/8) \log^2(t) 
    \log^{6}\frac{48 n}{\delta}.
\end{align*}
Once again $\sigma_2^{t^{\star}} \leq (tnm)^{-2}$ ensures $\sigma_2^{t^{\star}}\eta t \leq (1 \vee \eta t^{\star})$. Meanwhile we have $(g(\lambda_2,m))^2 \lesssim (\eta t)^{\gamma} / m$, $(g(\lambda_3,m))^2 \lesssim 1/m$ and with high probability $(f(\lambda^{\prime},nm,\delta/8))^2 \lesssim 1/(nm)$. As such to ensure this term is sufficiently small we require $\frac{ (\eta t)^{2 +\gamma} ( 1 \vee \eta t^{\star})^2}{m^2 (nm)} \leq (nm)^{-2r/(2r+\gamma)}$, which satisfied if $m \geq n^{\frac{1}{2r+\gamma}} (1 \vee (\eta t^{\star}))^{\frac{2r+\gamma}{2r+   \gamma - 1}}$. This then being implied by $m \geq (1 \vee \eta t^{\star})^{\frac{ (1+\gamma)(2r+\gamma)}{2(r+\gamma-1)}} n^{\frac{(r+1)}{(r+\gamma - 1)}}$ since $\frac{r+1}{r+\gamma -1} \geq \frac{1}{2r+\gamma}$ and $\frac{ (1+\gamma) (2r+\gamma)}{2(r+\gamma -1 )} \geq \frac{2r+\gamma}{2r+\gamma - 1}$. The second inequality arising from the observation that $\frac{1}{2(r+\gamma - 1)} \geq \frac{1}{2(r+\gamma - 1) + 1 - \gamma } = \frac{1}{2r+\gamma -1} $.

Each of the bounds for $\textbf{E}_i^2$ for $i=1,\dots,5$ hold in high probability, and as such, can be combined with a union bound. This incurs at most a logarithmic factor in the bound, with the number of unions applied being upper bounded by the constant $c_{\text{union}} > 1$ chosen at the start.

\end{proof}

\subsection{Worst Case (Theorem \ref{thm:WorstCase})}
\label{sec:proof:WorstCast}
Consider the refined bound in Theorem \ref{thm:Refined} with $r=1/2$ and $\gamma = 1$. 

\subsection{Leading Order Error Terms (Theorem \ref{thm:LeadingOrder})}
\label{sec:proof:LeadingOrder}
Follow the proof of Theorem \ref{thm:Refined}, where the error is decomposed into the following terms 
\begin{align*}
    \mathcal{E}(f_{t+1,v}) - \mathcal{E}(f_{\mathcal{H}})
    \leq 
    (\text{Network Error})^2 + (\text{Statistical Error})^2.
\end{align*}
The statistical error follows \cite{carratino2018learning} and, in our work, is summarised within Lemma \ref{lem:StatisticalError} to be upto logarithmic factors in high-probability 
\begin{align*}
    (\text{Statistical Error})^2
    \lesssim 
    \underbrace{ 
    \big( 1\vee \frac{\eta t }{M} \big)
    \frac{(\eta t)^{\gamma}}{nm}}_{\text{Sample Variance}}
    + 
    \underbrace{ 
    \frac{1}{M ( \eta t)^{(1 - \gamma)(2r- 1)}}}_{\text{Random Fourier Error}}
    + 
    \underbrace{ 
    \frac{1}{(\eta t)^{2r} }}_{\text{Bias}}.
\end{align*}
Meanwhile the network error is bounded into terms 
\begin{align*}
    (\text{Network Error})^{2}
    \lesssim
    \textbf{E}_1^2
    + 
    \textbf{E}_2^2
    + 
    \textbf{E}_3^2
    + 
    \textbf{E}_4^2
    + 
    \textbf{E}_5^2
\end{align*}
where high-probability bounds from Section \ref{sec:NetworkErrorBounds} are used. In particular, the bounds each term are, up to logarithmic factors, in high probability
\begin{align*}
    \textbf{E}_1^2 
    & \lesssim 
    \frac{ (\eta t^{\star})^{\gamma}}{m} \\
    \textbf{E}_2^2  
    & \lesssim
    \frac{ (\eta t^{\star})^2 (\eta t)^{\gamma} }{m^2} \\
    \textbf{E}_3^2 
    & 
    \lesssim 
    \frac{ (\eta t)^{2} (\eta t^{\star})^{1 + \gamma}}{m^2} \\
    \textbf{E}_{4}^2 
    & \lesssim
    \frac{ n \sigma_2^{2t^{\star}} (\eta t)^{2} }{m} \\
    \textbf{E}_{5}^2
    & \lesssim 
    \frac{ (\eta t)^{2} (1 \vee (\eta t^{\star}))^{\gamma}}{m(nm)}
    + 
    \frac{(\eta t)^{2+\gamma}(1 \vee \eta t^{\star})^2}{m^2 (nm)}
\end{align*}
The leading order terms are then defined as $\textbf{E}_{1}^{2}$ and $\textbf{E}_{3}^{2}$.

\section{Proofs of Auxiliary Lemmas}
In this section we provide the proofs of the auxiliary lemmas. This section is then as follows. Section \ref{sec:lem:conc} provides the proof for Lemma \ref{lem:conc}. Section \ref{sec:proof:lem:OperatorNormBound:Delta1} provides the proof of Lemma \ref{lem:OperatorNormBound:Delta1}. Section \ref{sec:proof:lem:OperatorNormBound:Delta2} provides the proof of Lemma \ref{lem:OperatorNormBound:Delta2}. 
\subsection{Concentration of Error terms (Lemma \ref{lem:conc})}
\label{sec:lem:conc}

\begin{proof}[Lemma \ref{lem:conc}]
Fix $w \in V$. We begin by collecting the necessary concentration results. Following Lemma 18 in \cite{lin2018optimal} with $\mathcal{T}_{\rho}, \mathcal{T}_{\mathbf{x}}$  swapped for $C_{M},\widehat{C}_{M}^{(w)}$ respectively (or Proposition 5 in \cite{rudi2017generalization})  we have with probability greater than $1-\delta$ 
\begin{align*}
    \|C_{M,\lambda}^{-1/2}( C_{M} - \widehat{C}_{M}^{(w)})\| 
    \leq 
    2 \kappa \Big( \frac{2 \kappa}{m \sqrt{\lambda}} + \sqrt{\frac{ \mathcal{N}_{M}(\lambda)}{m}}\Big) \log \frac{2}{\delta}
\end{align*}
From Lemma 2 in \cite{carratino2018learning} under assumptions \ref{ass:FeatureRegularity} and \ref{ass:RKHS} we have with probability greater than $1-\delta$ for all $t \geq 1$
\begin{align*}
    \|\widetilde{v}_{t+1}\| \leq 2 R \kappa^{2r -1}
    \Big( 1 + \sqrt{ \frac{9 \kappa^2 }{M} \log \frac{M}{\delta}}\max \Big( \sqrt{\eta t},\kappa^{-1}\Big) \Big).
\end{align*}
Meanwhile from Lemma 6 in \cite{rudi2017generalization} under assumption \ref{ass:FeatureRegularity} and \ref{ass:moment} we have with probability greater than $1-\delta$
\begin{align*}
    \|C_{M,\lambda}^{-1/2}(\widehat{S}_{M}^{(w) \top}\widehat{y} - S^{\star}_{M}f_{\rho}) \| 
    \leq 
    2\sqrt{ B} \Big( \frac{  \kappa}{\sqrt{\lambda}m} + \sqrt{ \frac{ 2 \sqrt{p} \mathcal{N}_{M}(\lambda) }{m}}\Big)\log \frac{2}{\delta}
\end{align*}
Considering $\|C_{M,\lambda}^{-1/2}N_{k,w}\|$, using triangle inequality and plugging the above bounds with a union bound, we have with probability greater than $1-\delta$
\begin{align*}
    & \|C_{M,\lambda}^{-1/2}N_{k,w}\|
     \leq 
    \|C_{M,\lambda}^{-1/2}( C_{M} - \widehat{C}_{M}^{(w)})\| \|\widetilde{v}_{t+1}\| 
    + 
    \|C_{M,\lambda}^{-1/2}(\widehat{S}_{M}^{(w) \star}\widehat{y} - S^{\star}_{M}f_{\rho})\| \\
    & \leq 
    2 \kappa \Big( \frac{2 \kappa}{m \sqrt{\lambda}} + \sqrt{\frac{ \mathcal{N}_{M}(\lambda)}{m}}\Big) \log \frac{6}{\delta}
    \Big( 1 + \sqrt{ \frac{9 \kappa^2 }{M} \log \frac{3 M}{\delta}}\max \Big( \sqrt{\eta t},\kappa^{-1}\Big) \Big) \\
    & \quad + 
    2 \sqrt{B} \Big( \frac{\kappa}{\sqrt{\lambda}m} + \sqrt{ \frac{2 \sqrt{p} \mathcal{N}_{M}(\lambda) }{m}}\Big)\log \frac{6}{\delta}.
\end{align*}
Now a bound over the maximum $\max_{w \in V} \|C_{M,\lambda}^{-1/2}N_{k,w}\|$ is obtained by taking a union bound over $w \in V$. Meanwhile, an identical set of steps with $\widehat{C}_{M}^{(w)}, \widehat{S}^{(w), \top}_{M}$ swapped for $\widehat{C}_{M}, \widehat{S}_{M}$ yields the bound for $\|C_{M,\lambda}^{-1/2}N_{k}\|$ and $\|C_{M,\lambda}^{-1/2}( C_{M} - \widehat{C}_{M})\|  $. 
\end{proof}

\subsection{Difference between Product of Empirical and Population Operators (Lemma \ref{lem:OperatorNormBound:Delta1})}
\label{sec:proof:lem:OperatorNormBound:Delta1}
In this section we provide the proof for Lemma \ref{lem:OperatorNormBound:Delta1}.
\begin{proof}[Lemma \ref{lem:OperatorNormBound:Delta1}]
Begin by writing the quantity $\Pi^{\Delta}(w_{t:1})N$ using two auxiliary sequences. Initialized at $\gamma_1 = \gamma^{\prime}_1 = N$ and updated for $t \geq s \geq 1$ we have
\begin{align*}
    \gamma^{\prime}_{s+1} & = (I - \eta \widehat{C}_{M}^{(w_{s})}) \gamma^{\prime}_s = \Pi(w_{s:1}) N \\
    \gamma_{s+1} & = (I - \eta C_{M} ) \gamma_s  =  (I - \eta C_{M})^{s} N
\end{align*}
We can then write the difference as between these two sequences  as the recursion
\begin{align*}
    & \gamma^{\prime}_{s+1} - \gamma_{s+1}
     = 
    (I - \eta C_{M}) (\gamma^{\prime}_{s} - \gamma_{s})
    + 
    \eta \big\{ C_{M} - \widehat{C}_{M}^{(w_{s})} \big\}
    \gamma^{\prime}_{s}\\
    & = 
    (I - \eta C_{M} )^{s}(\gamma^{\prime}_{1} - \gamma_{1}) 
    + 
    \sum_{\ell=1}^{s} \eta (I - \eta C_{M})^{s-\ell}  
    \big\{ C_{M} - \widehat{C}_{M}^{(w_{\ell})} \big\} 
    \gamma^{\prime}_{\ell}\\
    & = 
    \sum_{\ell=1}^{s} \eta (I - \eta C_{M})^{s-\ell} 
    \big\{ C_{M} - \widehat{C}_{M}^{(w_{\ell})} \big\} \gamma^{\prime}_{\ell}.
\end{align*}
We then have
\begin{align*}
    \| C_{M}^{1/2-u} \Pi^{\Delta}(w_{t:1}) N \|
    & = 
    \|C_{M}^{1/2-u}(\gamma_{t+1}^{\prime} - \gamma_{t+1})\| \\
    &  = 
    \| \sum_{\ell=1}^{t} \eta 
    C_{M}^{1/2-u} 
    (I - \eta C_{M})^{t-\ell} 
    \big\{ C_{M} - \widehat{C}_{M}^{(w_{\ell})} \big\} 
    \gamma^{\prime}_{\ell}\| \\
    & \leq 
    \sum_{\ell=1}^{t} \eta \| C_{M}^{1/2-u} 
    (I - \eta C_{M})^{t-\ell}  
    C_{M,\lambda}^{1/2} \| 
    \| C_{M,\lambda}^{-1/2}( C_{M} - \widehat{C}_{M}^{(w_{\ell})} )\| \|\gamma_{\ell}^{\prime}\| \\
    & \leq 
    \Delta_{\lambda} \|N\| \sum_{\ell=1}^{t} \eta 
    \|C_{M}^{1/2-u}  (I - \eta C_{M})^{t-\ell}  
    C_{M,\lambda}^{1/2} \|
\end{align*}
where we have taken out the maximum over the $w_{\ell} \in V$ for
$\| C_{M,\lambda}^{-1/2}( C_{M,\lambda} - \widehat{C}_{M}^{(w_{\ell})} )\|$ and simply bounded $\|\gamma_{\ell}^{\prime}\| = \|(I - \eta \widehat{C}_{M}^{(w_{\ell-1})} ) \gamma^{\prime}_{\ell-1}\| \leq \|\gamma^{\prime}_{\ell-1}\|\leq \|N\|  $ from $\eta \kappa^2 \leq 1$. 
\end{proof}

\subsection{Convolution of Difference between Product of Empirical and Population Operators (Lemma \ref{lem:OperatorNormBound:Delta2})}
\label{sec:proof:lem:OperatorNormBound:Delta2}
This section provides the proof of Lemma \ref{lem:OperatorNormBound:Delta2}. 

\begin{proof}[Lemma \ref{lem:OperatorNormBound:Delta2}]
Begin by observing that this quantity can be written as 
\begin{align*}
    \sum_{w_{t:1} \in V^{t}} \Delta(w_{t:1}) \Pi^{\Delta}(w_{t:1}) N
    & = \sum_{w_{t:1} \in V^{t}} \Delta(w_{t:1}) \Pi(w_{t:1})  N
    - 
    \sum_{w_{t:1} \in V^{t}} \Delta(w_{t:1}) (I - \eta C_{M})^{t} N\\
    & = \sum_{w_{t:1} \in V^{t}} \Delta(w_{t:1})
    \Pi(w_{t:1}) N
\end{align*}
since $\sum_{w_{t:1} \in V^{t}} \Delta(w_{t:1}) = 0$.  Now introduce the following auxiliary variables. Initialized as $\gamma_{1,w} = \gamma_{1,w}^{\prime} = N$ for all $w \in V$ we update the sequences for $t \geq s \geq 1$
\begin{align}
\label{equ:AuxiliarySequ}
    \gamma_{s+1,v} & = \sum_{w \in V} P_{vw} 
    (I - \eta \widehat{C}_{M}^{(w)}) \gamma_{s,w} = 
    \sum_{w_{s:1} \in V^{s}} P_{v w_{s:1}} \Pi(w_{s:1}) N \\
    \gamma_{s+1,v}^{\prime} & = 
    \sum_{w \in V} \frac{1}{n} (I - \eta \widehat{C}_{M}^{(w)} ) \gamma^{\prime}_{s,w} =  
    \sum_{w_{s:1} \in V^{s}}  \frac{1}{n^{s}} \Pi(w_{s:1})  N
    \nonumber . 
\end{align}
The quantity bounded within Lemma \ref{lem:OperatorNormBound:Delta2} can then be seen as the difference
\begin{align*}
    \| C_{M}^{1/2}(\gamma_{t+1,v}-\gamma_{t+1,v}^{\prime})\| 
    = 
    \Big\| 
    \sum_{w_{t:1} \in V^{t}} 
    \Delta(w_{t:1})
    C_{M}^{1/2}  \Pi(w_{t:1})N 
    \Big\|.
\end{align*}
Introducing the auxiliary sequence $\{\gamma^{\prime}_{s}\}_{s \geq 1}$ independent of the agents. Also initialised $\gamma_{1,w}^{\prime} = N =: \gamma^{\prime}_{1}$ for all $w \in V$ we have due to averaging over all of the agents uniformly $\gamma_{2,w}^{\prime} = \gamma^{\prime}_{2} = (I - \eta \widehat{C}_{M})N$ for all $w \in V$. Applying this recursively we have for $s \geq 1$ and $v \in V$
\begin{align*}
    \gamma_{s+1,v}^{\prime}  = \gamma^{\prime}_{s+1} =  (I - \eta \widehat{C}_{M})^{s}N. 
\end{align*}
Combined with the fact that the iterates $\{\gamma_{s,v}\}_{s \in [t],v\in V}$ can be written and unravelled
\begin{align*}
    & \gamma_{t+1,v} 
     = 
    \sum_{w \in V} P_{vw} \big( (I - \eta \widehat{C}_{M}) \gamma_{t,w} 
    + \eta \big\{ \widehat{C}_{M}  - \widehat{C}_{M}^{(w)} \big\} \gamma_{t,w}
    \big) \\
    & = 
    (I - \eta \widehat{C}_{M})^{t}N 
    + 
    \eta \sum_{k=1}^t  \sum_{w \in V} (P^{t-k+1})_{vw} (I - \eta \widehat{C}_{M})^{t-k}
    \big\{ \widehat{C}_{M} - \widehat{C}_{M}^{(w)} \big\} \gamma_{k,w},
\end{align*}
means the difference is written as 
\begin{align*}
    \gamma_{t+1,v} - \gamma_{t+1,v}^{\prime} 
    = 
    \eta \sum_{k=1}^t  \sum_{w \in V} (P^{t-k+1})_{vw} (I - \eta \widehat{C}_{M})^{t-k}
    \big\{ \widehat{C}_{M} - \widehat{C}_{M}^{(w)} \big\} \gamma_{k,w}.
\end{align*}
To analyse the difference $\gamma_{t+1,v} - \gamma_{t+1,v}^{\prime}$ we then consider the following decomposition where we denote the network averaged iterates $\overline{\gamma}_{t} = \frac{1}{n} \sum_{w \in V} \gamma_{t,w}$ 
\begin{align}
\label{equ:AuxiliarySequDecomp}
    \|C_{M}^{1/2}(\gamma_{t+1,v} - \gamma_{t+1,v}^{\prime})\|
    \leq 
    \underbrace{ 
    \|C_{M}^{1/2}(\gamma_{t+1,v} - \overline{\gamma}_{t+1})\|
    }_{\textbf{Term 1}}
    + \underbrace{ 
    \|C_{M}^{1/2} (\overline{\gamma}_{t+1} - \gamma^{\prime}_{t+1}) \|}_{\textbf{Term 2}}
\end{align}
It is clear the network average can be written using the fact that the communication matrix $P$ is doubly stochastic i.e. $\sum_{v \in V} P^{t-k+1}_{vw} = 1$ as follows 
\begin{align*}
    \overline{\gamma}_{t+1} - \gamma^{\prime}_{t+1}
     = \frac{1}{n} \sum_{v \in V} \gamma_{t+1,v} - \gamma^{\prime}_{t+1} 
     = 
    \eta \sum_{k=1}^{t} \frac{1}{n} \sum_{w \in V}  (I - \eta \widehat{C}_{M} )^{t-k} 
    \big\{ \widehat{C}_{M} - \widehat{C}_{M}^{(w)}  \big\} \gamma_{k,w}.
\end{align*}
When taking the difference we then arrive at 
\begin{align*}
    \gamma_{t+1,v} - \gamma^{\prime}_{t+1} - ( \overline{\gamma}_{t+1} - \gamma^{\prime}_{t+1})
    = 
    \eta \sum_{k=1}^{t}  \sum_{w \in V} ( (P^{t-k+1})_{vw} - \frac{1}{n}) 
    (I - \eta \widehat{C}_{M} )^{t-k} 
    \big\{ \widehat{C}_{M} - \widehat{C}_{M}^{(w)} \big\}
    \gamma_{k,w}
\end{align*}
We can then bound \textbf{Term 1} with $\lambda_1 > 0 $
\begin{align*}
    & \|C_{M}^{1/2}(\gamma_{t+1,v} - \overline{\gamma}_{t+1} )\|\\
    & \leq 
    \eta \sum_{k=1}^{t} \sum_{w \in V} 
    |(P^{t-k+1})_{vw} - \frac{1}{n}|
    \| C_{M}^{1/2} (I - \eta \widehat{C}_{M})^{t-k} C_{M,\lambda_1}^{1/2}
    \|
    \|C_{M,\lambda_1}^{-1/2} 
    \big\{ \widehat{C}_{M}  - \widehat{C}_{M}^{(w)} \big\}\|
    \| \gamma_{k,w}\| \\
    & \leq 
    2 \eta \Delta_{\lambda_1} \|N\|
    \sum_{k=1}^{t} 
    \| C_{M}^{1/2} (I - \eta \widehat{C}_{M}^{1/2} )^{t-k}
    C_{M,\lambda_1}^{1/2} \|
    \big( 
    \sum_{w \in V} 
    |(P^{t-k+1})_{vw} - \frac{1}{n}|
    \big)\\
    & \leq 
    4 \eta \Delta_{\lambda_1}\|N\|
    \sum_{k=1}^{t} 
    \|C_{M}^{1/2} (I - \eta \widehat{C}_{M})^{t-k}
    C_{M,\lambda_1}^{1/2}\|
    ( \sigma_2^{t-k+1}\wedge 1 )
\end{align*}
where we have used that $\|\gamma_{s+1,v}\| \leq \sum_{w \in V} P_{vw} \|(I - \eta \widehat{C}_{M}^{(w)}) \gamma_{s,w}\| \leq  \sum_{w \in V} P_{vw} \|\gamma_{s,w}\| \leq \|N\|$ as well as 
\begin{align*}
    \|C_{M,\lambda_1}^{-1/2}( \widehat{C}_{M}   - \widehat{C}_{M}^{(w)}) \|
    & \leq 
    \|C_{M,\lambda_1}^{-1/2}(\widehat{C}_{M} - C_{M})\|
    + 
    \|C_{M,\lambda_1}^{-1/2}( C_{M}  - \widehat{C}_{M}^{(w)} )\|\\
    & \leq 
    \frac{1}{n}\sum_{v \in V}
    \|C_{M,\lambda_1}^{-1/2}(C_{M} - \widehat{C}_{M}^{(v)} )\|
    + 
    \|C_{M,\lambda_1}^{-1/2}( C_{M}  - \widehat{C}_{M}^{(w)} )\|\\
    & \leq 
    2 \Delta_{\lambda_{1}}
\end{align*}
in addition to Lemma \ref{lem:SpectralBound} to bound $\sum_{w \in V}  |(P^{t-k+1})_{vw} - \frac{1}{n}| = \sum_{w \in V} | \Delta^{t-k+1}(v,w)| $.

To bound \textbf{Term 2} we note that we can rewrite 
\begin{align*}
    \overline{\gamma}_{t+1}  - \gamma_{t+1}^{\prime}
    & = 
    \eta 
    \sum_{k=2}^{t} \frac{1}{n} \sum_{w \in V}  (I - \eta \widehat{C}_{M} )^{t-k} 
    \big\{ \widehat{C}_{M}  - \widehat{C}_{M}^{(w)} \big\} (\gamma_{k,w}-\overline{\gamma}_{k}).
\end{align*}
where $\frac{1}{n} \sum_{w \in V}(I - \eta \widehat{C}_{M} )^{t-k} \big\{ \widehat{C}_{M}   - \widehat{C}_{M}^{(w)}  \big\} \overline{\gamma}_{k}  = 0$ for $k \geq 1$. Applying triangle inequality as well as similar step to previously, we get with $ \lambda_2,\lambda_3 \geq 0$ 
\begin{align*}
    & \|C_{M}^{1/2}( \overline{\gamma}_{t+1} - \gamma_{t+1}^{\prime}) \|\\
    & \leq
    \eta \sum_{k=2}^{t} \|C_{M}^{1/2}  (I - \eta \widehat{C}_{M})^{t-k}
    C_{M,\lambda_2}^{1/2}\|
    \frac{1}{n} \sum_{w \in V}
    \|C_{M,\lambda_2}^{-1/2}( \widehat{C}_{M}  - \widehat{C}_{M}^{(w)} )\|
    \|\gamma_{k,w}-\overline{\gamma}_{k}\| \\
    & \leq 
    8  \eta^2 \Delta_{\lambda_{2}}\Delta_{\lambda_3} \|N\|
    \sum_{k=2}^{t} \sum_{\ell=1}^{k-1} 
    \|C_{M}^{1/2} (I - \eta \widehat{C}_{M} )^{t-k}C_{M,\lambda_2}^{1/2}\|
    \| (I - \eta \widehat{C}_{M} )^{k-1-\ell} C_{M,\lambda_{3}}^{1/2} \|
    (\sigma_2^{k - \ell} \wedge 1) 
\end{align*}
where we plugged in the bound from \textbf{Term 1} for the deviation $\|\gamma_{k,w}-\overline{\gamma}_{k}\|$ for $k \geq 2$.
\end{proof}

\end{document}